\documentclass{article}

\usepackage{microtype}
\usepackage{graphicx}
\usepackage{subfigure}
\usepackage{booktabs} 
\usepackage{nicematrix}
\usepackage{natbib}
\usepackage{hyperref}

\usepackage[accepted]{icml2025}

\usepackage{amsmath}
\usepackage{amssymb}
\usepackage{mathtools}
\usepackage{amsthm}
\allowdisplaybreaks
\newcommand*{\Scale}[2][4]{\scalebox{#1}{$#2$}}%
\usepackage[capitalize,noabbrev]{cleveref}

\theoremstyle{plain}
\newtheorem{theorem}{Theorem}

\newtheorem{lemma}{Lemma}

\theoremstyle{definition}

\newtheorem{assumption}{Assumption}
\theoremstyle{remark}
\newtheorem{remark}{Remark}

\usepackage[textsize=tiny]{todonotes}

\icmltitlerunning{Widening the Network Mitigates the Impact of Data Heterogeneity on FedAvg}

\makeatletter
\def\th@remark{%
	\thm@headfont{\bfseries}
	\itshape
}
\makeatother

\begin{document}

\twocolumn[
\icmltitle{Widening the Network Mitigates the Impact of Data Heterogeneity on FedAvg}



\icmlsetsymbol{equal}{*}

\begin{icmlauthorlist}
\icmlauthor{Like Jian}{yyy}
\icmlauthor{Dong Liu}{yyy}
\end{icmlauthorlist}

\icmlaffiliation{yyy}{School of Cyber Science and Technology, Beihang University, Beijing, China}

\icmlcorrespondingauthor{Dong Liu}{dliu@buaa.edu.cn}

\icmlkeywords{Machine Learning, ICML}

\vskip 0.3in
]



\printAffiliationsAndNotice{}  

\begin{abstract}
Federated learning (FL) enables decentralized clients to train a model collaboratively without sharing local data. A key distinction between FL and centralized learning is that clients' data are non-independent and identically distributed, which poses significant challenges in training a global model that generalizes well across heterogeneous local data distributions. In this paper, we analyze the convergence of overparameterized FedAvg with gradient descent (GD). We prove that the impact of data heterogeneity diminishes as the width of neural networks increases, ultimately vanishing when the width approaches infinity. In the infinite-width regime, we further prove that both the global and local models in FedAvg behave as linear models, and that FedAvg achieves the same generalization performance as centralized learning with the same number of GD iterations. Extensive experiments validate our theoretical findings across various network architectures, loss functions, and optimization methods.
\end{abstract}

\section{Introduction}
Federated Learning (FL) is a distributed machine learning paradigm that enables collaborative model training across distributed clients while preserving data locality~\cite{mcmahan2017communication}, a critical feature for privacy-sensitive domains such as healthcare, finance, and mobile computing, where regulatory or infrastructural constraints prohibit data centralization. However, a fundamental challenge in FL arises  from the intrinsic non-independent and identically distributed (non-IID) nature of client data~\cite{li2021fedbn}, where local datasets exhibit significant distributional shifts due to user-specific behaviors, geographic variations, or device-specific usage patterns. Such statistical heterogeneity leads to divergent local optimizations, degrading model convergence and generalization~\cite{li2019convergence, zhao2018federated}.

To address the challenges posed by non-IID data in FL, numerous research efforts have emerged,  including client regularization~\cite{li2020federated}, adaptive optimization frameworks~\cite{karimireddy2020scaffold, reddi2020adaptive}, personalized model architectures~\cite{jeong2022factorized, t2020personalized}, and etc. While these approaches have demonstrated empirical success, they often require intricate hyperparameter tuning or restrictive assumptions about convexity, data similarity, or gradient boundedness, limiting their applicability in practical highly heterogeneous environments.

In parallel, overparameterized neural networks have garnered prominence in centralized learning for their remarkable ability to achieve strong generalization despite non-convex optimization landscapes, underpinned by theoretical frameworks such as the neural tangent kernel (NTK)~\cite{jacot2018neural}. These networks exhibit implicit regularization properties, enabling  interpolation of complex data distributions while maintaining robust generalization~\cite{neyshabur2014search, neyshabur2019role, neyshabur2019towards, lee2018deep}, motivating  a pivotal question: Can increasing the network width inherently mitigate the effects of data heterogeneity in FL?

In this work, we analyze the convergence of FedAvg with gradient descent (GD) for multi-layer overparameterized neural networks and establish that the impact of data heterogeneity can  indeed  be reduced by widening the network. Further, we prove that as the network width approaches infinity, both global and local models behave as linear models. Strikingly, in this regime, FedAvg and centralized GD yield identical model parameters and outputs under matched iterations, achieving equivalent generalization performance. To the best of our knowledge, this is the first work to provide a quantitative analysis explicitly linking the width of neural networks to the impact of data heterogeneity on both FL training and generalization. Our key contributions are:
\vspace{-2mm}
\begin{itemize}  
	\item \textbf{Theoretical guarantees for heterogeneity reduction:} We prove that the model divergence is bounded and decreases inversely proportional to the square root of the network width asymptotically, without relying on restrictive assumptions on convexity or gradient similarity/boundedness. This bound is vital in the convergence analysis of FedAvg, revealing that the impact of data heterogeneity slows the convergence rate, but vanishes when the width approaches infinity, allowing the convergence rate to recover linearity. 
	\item \textbf{Bridging federated and centralized learning:} We extend the NTK theory from centralized learning to FL with multi-layer networks, showing that infinite network width induces constant global and local NTKs, further linearizes both global and local models. Notably, we prove the equivalence between infinite-width FedAvg and centralized GD, thereby achieving the same generalization performance, bridging decentralized and centralized learning paradigms.
	\item \textbf{Empirical validation:} We conducted numerous experiments on MINST and CIFAR-10 datasets, spanning diverse network structures, loss functions, and optimizers to validate our theoretical analysis.  
\end{itemize}

\section{Related Work}
\paragraph{Data Heterogeneity.} While prior works have provided convergence analyses of federated learning with non-IID data, many rely on restrictive assumptions. For example, some studies assume convex loss functions \cite{cho2020client, khaled2019first, li2019convergence}, others require bounded gradient dissimilarity \cite{li2020federated, zhang2023fed, wang2020tackling}, and some assume bounded gradients \cite{li2019convergence, cho2020client}. These conditions are often difficult to satisfy in practice. In this paper, we analyze the convergence of FedAvg via NTK without imposing convexity, bounded gradients, or bounded gradient dissimilarity. 

To address data heterogeneity, various techniques have been proposed. Regularization-based methods \cite{li2020federated, durmus2021federated} introduce a regularization term during local updates to mitigate client drift. Client selection approaches \cite{cho2020client, goetz2019active, zhang2023fed} choose a subset of clients whose aggregated updates approximate those of all clients. Personalized federated learning  \cite{jeong2022factorized, jiangheterogeneous} allows clients to leverage aggregated knowledge while fine-tuning on their local data. Other methods, such as SCAFFOLD \cite{karimireddy2020scaffold} and FedNova \cite{wang2020tackling}, correct optimization bias between clients and the global model to improve convergence. FedMA \cite{wang2020federated} constructs the global model layer by layer to diminish the impact of heterogeneous data, and MOON \cite{li2021model} compares local and global models to correct client drift.

\paragraph{Overparameterized FL.} 
Recent works have made substantial progress in overparameterized federated learning. For instance, \citet{li2021fedbn} proposed FedBN, which applies local batch normalization to mitigate heterogeneity, and analyzed its convergence using NTK. However, their analysis is limited to two-layer networks, limiting the model capacity. Under the same two-layer assumption, \citet{jiangheterogeneous} proposed a local-global update mixing method and analyzed its convergence via NTK, while \citet{huang2021fl} showed that overparameterized FedAvg converges to the global optimal solution with linear convergence rate. 

Moving beyond the depth constraint on networks, \citet{deng2022local} proved that overparameterized FedAvg with ReLU activation converges in polynomial time with stochastic gradient descent (SGD). Fed-ensemble \cite{shi2021fed} employs model ensembling to enhance the generalization of FL, supported by an NTK-based convergence analysis. Yet, data heterogeneity is not considered. To address data heterogeneity, \citet{yue2022neural} proposed to let clients transmit Jacobian matrices rather than weights or gradients, providing a more expressive data representation. \citet{yu2022tct}  observed that the learning of final layers in FL is strongly influenced by non-convexity and propose the train-convexify-train (TCT) method to alleviate these issues. 

Among the most relevant works, \citet{song2023fedavg} established that overparameterized FedAvg achieves linear convergence to zero training loss, and empirically observed that wide neural networks achieve better and more stable performance in FL. By contrast, we theoretically prove that the model divergence in FL caused by data heterogeneity is bounded by $\mathcal{O}(n^{-\frac{1}{2}})$, where $n$ is the network width, thereby establishing the first quantitative relationship between network width and the mitigation of heterogeneity. Beyond this, we unveil the generalization performance of overparameterized FL by proving that infinite-width FedAvg and centralized learning yield identical model outputs under matched training iterations, while the theoretical analysis in \cite{song2023fedavg} solely focused on training loss.

\section{Notations and Problem Formulation}\label{section3}
In this section, we establish the basic notations, formulate the FL problem, and introduce the metric quantifying the impact of data heterogeneity. 

We consider a standard FL setup consisting of a server and $M$ clients. The local training dataset of client $i$ is denoted by $\mathcal{D}_i$ with $\mathcal{X}_i = \left\{x|(x,y)\in \mathcal{D}_i\right\}$ and $\mathcal{Y}_i = \left\{y|(x,y)\in \mathcal{D}_i\right\}$ representing the set of inputs and labels of client $i$, respectively, where $x\in\mathbb{R}^{n_0}$ and $y\in\mathbb{R}^{k}$. The global dataset is defined as the union of all clients dataset, i.e., $\mathcal{D}=\cup_{i=1}^{M} \mathcal{D}_i$ with $\mathcal{X}=\cup_{i=1}^{M} \mathcal{X}_i$ and $\mathcal{Y}=\cup_{i=1}^{M} \mathcal{Y}_i$.

Suppose that each client trains a $L$-layer fully-connected neural network (FNN), where the width of the $l$-th layer is denoted by $n_l$. Then,  the output of the $l$-th layer of client $i$'s model can be expressed as
\begin{equation}
	f_{i,l}(x)=\left\{
	\begin{aligned}
		& x, && l=0 \\
		& \sigma\left(W_{i,l} f_{i, l-1}(x)+b_{i, l}\right), && 0<l<L \\
		& W_{i, L}f_{i, L-1}(x)+b_{i, L}, && l=L
	\end{aligned}
	\right. 
\end{equation}
where $W_{i,l} \in \mathbb{R}^{n_{l}\times n_{l-1}}$ and $b_{i,l}\in \mathbb{R}^{n_l}$ are the weight and bias of the $l$-th layer of client $i$, respectively, and $\sigma(\cdot)$ is the activation function. 

We define $\theta_i={\rm vec} \left(\{W_{i,l}, b_{i,l} \}_{l=1}^{L}\right) \in \mathbb{R}^{w}$ as the vector of all trainable parameters of client $i$'s model. Then, the model output of a single input sample $x$ can be expressed as $f_i(x, \theta_i)$ and the concatenated output of all samples in $\mathcal X_i$ is denoted by $f_i(\mathcal{X}_i, \theta_i) = {\rm vec} \big(\left\{f(x,\theta_i)\right\}_{x\in \mathcal X_i}\big) \in \mathbb{R}^{k|\mathcal{D}_i|}$. Similarly, we further define the global model $f$ having the same structure as $f_i$ but with parameter $\theta \in \mathbb{R}^{w}$. Analogously, its concatenated output of all samples in $\mathcal X$ is  denoted by $f(\mathcal{X}, \theta) = {\rm vec} \big(\left\{f(x,\theta)\right\}_{x\in \mathcal X}\big) \in \mathbb{R}^{k|\mathcal{D}|}$.  To simplify the notations, we use the short hand $f_i(\theta_i) \triangleq f_i(\mathcal{X}_i, \theta_i)$ and $f(\theta) \triangleq f(\mathcal{X}, \theta)$  in the following.

We consider the mean square error (MSE) loss function, and hence the loss function of client $i$ is expressed as
\begin{equation}
	\Phi_i=\frac{1}{2|\mathcal{D}_i|}\sum_{\left(x,y\right)\in\mathcal{D}_i}\left\| f_i\left(x, \theta_i\right)-y \right\|_2^2,
\end{equation}
The goal is to minimize the global loss function defined by $\Phi=\sum_{i=1}^{M}p_i\Phi_i$, where $p_i = \frac{|\mathcal{D}_i|}{|\mathcal{D}|}$.

At initialization, each client's model parameters are sampled from the Gaussian distribution as follows
\begin{align}
	W_{1,l}^0 &= \cdots = W_{M,l}^0   \sim \mathcal{N}\big(0,\tfrac{\sigma_{W_l}^2}{n_l}\big), \\ 
	b_{1, l}^0 &= \cdots = b_{M, l}^0 \sim \mathcal{N}\left(0,\sigma_{b}^2\right).
\end{align}
Upon initialization, each client updates its local model minimizing the loss function by GD for $\tau$ iterations. For every $\tau$ local iterations, each client uploads its local model to the server for model aggregation, and then the server broadcasts the aggregated model to each client for the next round. Let $t$ denote the number of global rounds. Then, the model parameters of client $i$ after the $r$-th ($1\leq r \leq \tau $) local iteration in the $t$-th global round can be denoted by $\theta^{t\tau + r}_i$. 

Specifically, during the $t$-th and $(t+1)$-th global round, say in the $(t\tau +r + 1)$-th total iteration, the model parameters are updated by GD as
\begin{equation}\label{eq:equation13}
	\theta_i^{t\tau+r + 1} \leftarrow \theta_i^{t\tau+r }-\frac{\eta}{|\mathcal{D}_i|} J_i\left(\theta_i^{t\tau+r}\right)g_i\left(\theta_i^{t\tau+r}\right)
\end{equation}
where $\eta$ is the learning rate and
\begin{align}
J_i\left(\theta_i\right) &=  \nabla_{\theta_i} f_i\left(\theta_i\right) \in \mathbb{R}^{w\times k|\mathcal D_i|}, \\
g_i\left(\theta_i\right) &= f_i\left(\theta_i\right) - {\rm vec}(\mathcal Y_i) 
\end{align}
are the local Jacobian matrix and error vector, respectively. Similarly, we define the global Jacobian matrix and error vector respectively, as
\begin{align}
	J\left(\theta\right) & \triangleq  \nabla_{\theta} f\left(\theta\right) \in\mathbb{R}^{w\times k|\mathcal D|}, \\
	g\left(\theta\right) & \triangleq  f\left(\theta\right)-{\rm vec}(\mathcal{Y}).
\end{align}
In the $(t+1)$-th global round, the model is aggregated by FedAvg, i.e., 
\begin{equation}\label{eq:equation14}
	\theta^{\left(t+1\right)\tau}=\sum_{i=1}^{M}p_i\theta_i^{t\tau+\tau},
\end{equation}
where we let $t\tau+\tau$ and $(t+1)\tau$ to denote the time instants  before and after the $(t+1)$-th global aggregation, respectively. 
Then, the aggregated parameters $\theta^{(t+1)\tau}$ are broadcast to all clients, which yields $\theta_i^{(t+1)\tau} = \theta^{(t+1)\tau}, \forall i$.
Consequently, the relation between the global and local Jacobian and error at the $t$-th global round can be described by
\begin{align}
	g_i\left(\theta_i^{t\tau}\right)&  = P_ig\left(\theta^{t\tau}\right), \\
J\left(\theta^{t\tau}\right)& =\sum_{i=1}^{M}J_i\left(\theta_i^{t\tau}\right)P_i,
\end{align}
where $P_i\in \mathbb{R}^{k|\mathcal{D}_i|\times {k|\mathcal{D}|}}$ is a projection matrix defined as
\begin{equation}
	P_i =
	\begin{pNiceArray}{cccccccccc}[margin=1pt]
	0       & \cdots &       0 & 1 & 0 & \cdots & 0 & 0 & \cdots & 0 \\
    0       & \cdots &      0 & 0 & 1 & \cdots & 0 & 0 & \cdots & 0 \\
    \vdots  & \vdots & \vdots & \vdots & \vdots & \ddots & \vdots & \vdots & \vdots & \vdots \\
    0       & \cdots & 0 & 0 & 0 & \cdots & 1 & 0 & \cdots & 0 \\
    \CodeAfter
     \UnderBrace{4-1}{4-3}{ \Scale[0.6]{k|\mathcal D_i|\times} \Scale[0.5]{\sum_{j=1}^{i-1}} \Scale[0.6]{k|\mathcal D_j|}}[shorten]
    \UnderBrace{4-4}{4-7}{\Scale[0.6]{k|\mathcal D_i| \times k|\mathcal D_i|}}[shorten]
     \UnderBrace{4-8}{4-10}{ \Scale[0.6]{k|\mathcal D_i|\times} \Scale[0.5]{\sum_{j=i+1}^{M}} \Scale[0.6]{k|\mathcal D_j|}}[shorten]
	\end{pNiceArray},
	\vspace{3mm}
\end{equation}
whose operator norm $\|P_i\|_{\rm op} = 1$.

To facilitate convergence analysis, the following notations are also introduced: 
\begin{align}
	f\left(\theta_i^{t\tau+r}\right)&= f\left(\mathcal{X}, \theta_i^{t\tau+r}\right), \\
	J\left(\theta_i^{t\tau+r}\right)&= \nabla_{\theta_i^{t\tau+r}} f\left(\theta_i^{t\tau+r}\right), \\
	g\left(\theta_i^{t\tau+r}\right)&= f\left(\theta_i^{t\tau+r}\right)-{\rm vec}(\mathcal{Y}),
\end{align}
where $f\left(\theta_i^{t\tau+r}\right)$ denotes the output of the global model when its parameters are replaced with client $i$'s parameters in round $t\tau+r$.

In the $(t+1)$-th global round, the degree to which client $i$'s model deviates from the global model is characterized by
\begin{equation}
	\left\|\Delta \theta_i^{(t+1)\tau}\right\|_2= \left\|\theta_i^{t\tau+\tau} - \theta^{(t+1)\tau}\right\|_2 \label{eqn:di}
\end{equation}
Therefore, we use $\sum_{i=1}^{M}p_i\big\|\Delta  \theta_i^{(t+1)\tau}\big\|_2$  to quantify the degree of data heterogeneity and term it as \emph{model divergence}. Apparently, when the data is IID, $\sum_{i=1}^{M}p_i\big\|\Delta  \theta_i^{(t+1)\tau}\big\|_2$ approaches zero as the number of local data increases. By contrast, when the data is non-IID, $\sum_{i=1}^{M}p_i\big\|\Delta  \theta_i^{(t+1)\tau}\big\|_2$ remains non-zero and increases with the degree of data heterogeneity.

\section{Convergence Analysis}
In this section, we analyze the convergence of overparameterized FedAvg. We derive the  bound on the model divergence explicitly  and analyze how it influences the convergence rate and error.

We first introduce several notations regarding  overparameterized neural networks. Let $n =\min \{n_1,n_2,\cdots,n_L\}$ and define the global NTK matrix \cite{lee2019wide}  in the $t$-th global round as
\begin{equation} \label{eqn:ntk}
	\Theta^{t\tau}=\frac{1}{n}J\left(\theta^{t\tau}\right)^TJ\left(\theta^{t\tau}\right).
\end{equation}
Meanwhile, the analytic NTK matrix is defined as
\begin{equation}\label{eq:equation17}
	\Theta = \lim_{n\to\infty}\Theta^{0}.
\end{equation}
Analogously, the local NTK matrix of the standard parameterization in the $(t\tau+r)$-th iteration is defined as
\begin{equation}\label{equation21}
	\Theta_i^{t\tau+r}=\frac{1}{n}J\left(\theta_i^{t\tau+r}\right)^TJ_i\left(\theta_i^{t\tau+r}\right)P_i.
\end{equation}

The following assumptions are made to facilitate  the convergence analysis. 
\begin{assumption}\label{ass:wide}
    	The minimum width among all hidden layers $n$ is sufficiently large such that the terms of order $\mathcal{O}(n^{-1})$ and higher are omitted.
\end{assumption}
\begin{assumption}\label{ass:ntkfullrank}
	 The analytic NTK $\Theta$ is full rank, i.e., the minimum eigenvalue $\lambda_{m}$ of $\Theta$ satisfies $\lambda_m>0$.
\end{assumption}
\begin{assumption}\label{ass:compact}
	The norm of every input data is bounded, i.e., $\|x\|_2 \leq 1$.
\end{assumption}
\begin{assumption}\label{ass:activation}
	The activation function $\sigma$ satisfies \[
	|\sigma(0)|, \|\sigma'\|_{\infty}, \sup_{\substack{x \neq x'}} \frac{|\sigma(x)-\sigma(x')|}{|x-x'|} < \infty
	\]
\end{assumption}
Assumptions~\ref{ass:wide} $\sim$ \ref{ass:activation} are common in analyzing the overparameterized neural network  \cite{lee2019wide, shi2021fed}.

The learning rate is set to $\eta = \frac{\eta_0}{n}$ and $\eta_0$ is a constant independent of $n$, which results in infinitesimally updates during each gradient descent step when $n$ is sufficiently large. Consequently, we adopt gradient flow as an approximation of gradient descent, which can be expressed as
\begin{equation}
	\frac{d\theta_i^{t\tau+r}}{dr}=-\frac{\eta}{|\mathcal{D}_i|} J_i\left(\theta_i^{t\tau+r}\right) g_i\left(\theta_i^{t\tau+r}\right) .
\end{equation}
Next, we present the main theorem regarding the bound of model divergence and the convergence of overparameterized FedAvg.

\begin{theorem}\label{theorem1}
	Under Assumptions~\ref{ass:wide} to \ref{ass:activation}, for any small $\delta_0>0$, there exist $R_0>0$, $N>0$, $\eta_0>0$, $C>0$ and $C_1>0$, such that for any $n\geq N$, the following holds with probability at least $(1-\delta_0)$ over random initialization:
	\begin{align}
		&\sum_{i=1}^{M}p_i\left\|\Delta  \theta_i^{(t+1)\tau}\right\|_2 \leq \zeta \triangleq \frac{2\eta_0\tau CR_0}{\sqrt{n}\left(1-q\right)}, \label{eq:het}\\
		&\left\|g(\theta^{t\tau})\right\|_2\leq q^{t}R_0+\frac{2\eta_0\tau CC_1R_0\zeta\left(1-q^{t}\right)}{\left(1-q\right)^2}, \label{eq:rate}\\
		& \left\|\theta^{t\tau}-\theta^0\right\|_2 \leq \frac{\eta_0 \tau C R_0\left(1-q^t\right)}{\sqrt{n}\left(1-q\right)}, \label{eq:lazy}\\
		& \left\|\Theta^{t\tau}-\Theta^0\right\|_F\leq \frac{2\eta_0 \tau C^3 R_0\left(1-q^t\right)}{\sqrt{n}\left(1-q\right)}, \label{eq:ntk}\\
		& \left\|\Theta_i^{t\tau+r}\! - \Theta_i^0\right\|_F \nonumber\\
		&\!\leq\! \frac{2\eta_0 r q^{t}C^3R_0\sqrt{k}}{\sqrt{n|\mathcal{D}_i|}}+ \frac{2\eta_0 \tau C^3 R_0\left(1-q^t\right)\sqrt{k|\mathcal{D}_i|}}{\left(1-q\right)\sqrt{n}}\label{eq:local ntk}
	\end{align}
	where $q = 1-\frac{\eta_0\tau\lambda_m}{3|\mathcal{D}|}+\frac{\eta_0^2\tau^2C^4}{2}e^{\eta_0\tau C^2}$.
\end{theorem}
The  detailed proof is provided in appendix~\ref{appendixA} and we present the proof sketch in the following.

{\textbf{Proof Sketch.} We use mathematical induction to prove Theorem~\ref{theorem1} and the induction hypotheses are \eqref{eq:rate} and \eqref{eq:lazy}. It is trivial that \eqref{eq:rate} and \eqref{eq:lazy} hold when $t=0$, and our aim is to prove \eqref{eq:rate} and \eqref{eq:lazy} for $t+1$. 

[\textbf{Step 1}] We first present several essential lemmas in Appendix~\ref{lemmas}, including proving the Lipschitz continuity of the global and local Jacobians and some properties regarding the Taylor series expansion.

[\textbf{Step 2}] Prove  induction hypothesis \eqref{eq:lazy} holds for $t+1$.  Due to the small learning rate $\eta = \frac{\eta_0}{n}$ for large $n$, we treat time as continuous and use gradient flow to approximate GD, which yields
\begin{equation}\label{eq:equatin29}
	\frac{dg\left(\theta_i^{t\tau+r}\right)}{dr} = -\frac{\eta_0\Theta_i^{t\tau+r}}{|\mathcal{D}_i|} g\left(\theta_i^{t\tau+r}\right).
\end{equation}
Applying the mean value theorem of integral, we can obtain
\begin{equation}\label{eq:equation30}
	g\left(\theta_i^{t\tau+\tau}\right) = e^{-\frac{\eta_0\tau}{|\mathcal{D}_i|}\Theta_i^{t\tau+\bar{r}_i}} g\left(\theta_i^{t\tau}\right),\bar{r}_i\in\left(0,\tau\right).
\end{equation} 
Based on \eqref{eq:equation30} and the induction hypotheses, the local model parameters variation can be bounded by:
\begin{equation}\label{eq:equation31}
	\left\|\theta_i^{t\tau+\tau}-\theta_i^0\right\|_2 \leq \frac{\eta_0 \tau C \left\|g\left(\theta^{t\tau}\right)\right\|_{2}}{\sqrt{n}|\mathcal{D}_i|} + \left\|\theta^{t\tau}-\theta^0\right\|_2.
\end{equation}
Notably, \eqref{eq:equation31} captures the relationship between the dynamics of local model and global model, based on which we can obtain the recursive relationship between $\|\theta^{t\tau}-\theta^0\|_2$ and $\|\theta^{\left(t+1\right)\tau}-\theta^0\|_2$ from  
\begin{multline}\label{eq:equation29}
	\big\|\theta^{\left(t+1\right)\tau}-\theta^0\big\|_2 
	\leq \sum_{i=1}^{M}p_i\left\|\theta_i^{t\tau+\tau}-\theta^0\right\|_2\\
	\leq \sum_{i=1}^{M}\frac{\eta_0 \tau C \left\|g\left(\theta^{t\tau}\right)\right\|_{2}}{\sqrt{n}} + \left\|\theta^{t\tau}-\theta^0\right\|_2
\end{multline}
Using \eqref{eq:equation29}, we can prove \eqref{eq:lazy} holds for $t+1$.

[\textbf{Step 3}] Based on the inductions hypotheses and a variation of \eqref{eq:equation31}, we are able to obtain \eqref{eq:het}.

[\textbf{Step 4}] Prove induction hypothesis \eqref{eq:rate} holds for $t+1$. By taking the Taylor series expansion of  $\sum_{i=1}^{M}p_ig\left(\theta_i^{t\tau+\tau}\right)$ at $\theta^{\left(t+1\right)\tau}$, we are able to derive
\begin{equation}\label{eqn:taylor}
	\!\left\|g\big(\theta^{\left(t+1\right)\tau}\big)\right\|_2\! \leq  \left\|\sum_{i=1}^{M}p_ig\left(\theta_i^{t\tau+\tau}\right)\right\|_2\! + \left\|\sum_{i=1}^{M}p_i\Omega_i\right\|_2,\!
\end{equation}
where $\Omega_i$ represents the remainder terms. Based on the results of [\textbf{Step 1}] $\sim$ [\textbf{Step 3}], we can further bound  $\|\sum_{i=1}^{M}p_ig(\theta_i^{t\tau+\tau})\|_2$ and $\|\sum_{i=1}^{M}p_i\Omega_i\|$, respectively, as
\begin{align} 
	\left\|\sum_{i=1}^{M}p_ig\left(\theta_i^{t\tau+\tau}\right)\right\|_2 & \leq q \left\|g\left(\theta_i^{t\tau}\right)\right\|_2 \label{eqn:bound1} \\ 
	\left\|\sum_{i=1}^{M}p_i\Omega_i\right\|_2 & \leq \frac{2\eta_0\tau CC_1R_0\zeta}{\left(1-q\right)} \label{eqn:bound2}
\end{align} 
Plugging \eqref{eqn:bound1} and \eqref{eqn:bound2} into \eqref{eqn:taylor}, \eqref{eq:rate} holds for $t+1$. 

[\textbf{Step 5}] Based on the Lipschitzness of the global and local jacobians as well as the results in [\textbf{Step 2}]$\sim$[\textbf{Step 4}], we prove \eqref{eq:ntk} and \eqref{eq:local ntk}.

\begin{remark}[\bf Bound on the model divergence]
	  Inequation \eqref{eq:het} establishes an upper bound $\zeta$ on the model divergence caused by data heterogeneity. Since $\zeta = \mathcal{O}(n^{-\frac{1}{2}})$, increasing the network width can reduce the effect of data heterogeneity. Note that we do not impose any strict assumptions on the convexity of the loss function \cite{cho2020client, khaled2019first, li2019convergence}, the bound of local gradients \cite{li2019convergence, cho2020client} or the divergence between local and global gradient \cite{li2020federated, zhang2023fed, wang2020tackling}. Instead, we prove that model the divergence is indeed bounded as long as the network is sufficiently wide.
\end{remark}
\begin{remark}[\bf Impact of data heterogeneity on the convergence rate]
	 Inequation \eqref{eq:rate} characterizes  the evolution of training error across the global aggregation rounds. Different from \citet{song2023fedavg, huang2021fl}, the presence of $\zeta>0$ here slows down the convergence, making the convergence rate no longer linear and the convergence error no longer zero. Recalling that $\zeta = \mathcal{O}(n^{-\frac{1}{2}})$, widening the network enhances the convergence rate by mitigating the model divergence. When $n\to\infty$, we have $\zeta\to 0$ and the impact of data heterogeneity vanishes, resulting in a linear convergence rate and zero training error as shown in \eqref{eq:rate}. 
\end{remark}
\begin{remark}[\bf Lazy training]
	Inequality (24) shows that as the network width increases, each global update in overparameterized FedAvg remains confined within an increasingly smaller neighborhood of size $\mathcal{O}(n^{-\frac{1}{2}})$ around its initialization, thereby extending the lazy-training phenomenon observed in centralized settings (Chizat et al., 2019) to FL settings.
\end{remark}
\begin{remark}[\bf Constant global and local NTKs]	 
Inequation \eqref{eq:ntk} shows that as the network width increases, the global and local NTK experiences less variation during training. When the width approaches infinity, the both the global and local NTKs are constant, which extends the findings in centralized learning \cite{jacot2018neural} to FL settings.
\end{remark}

Next, we will investigate the training dynamics of FedAvg in the infinite-width regime and compare it with centralized learning to further investigate the generalization performance of overparameterized FedAvg.

\section{Generalization Performance}
In this section, we analyze the training dynamics and generalization performance of overparameterized FedAvg as the network width $n\to\infty$. First, we prove that both the global and local models behave as linear models during the training process. Then, we derive the closed-form expression of those linear models and establish the equivalence between infinite-width FedAvg and centralized GD.

We define the linear models $f^{\rm lin}\left(\theta^{t\tau}\right)$ and $f_i^{\rm lin}\left(\theta_i^{t\tau+r}\right)$ as the first-order Taylor expansion of the global model $f\left(\theta^{t\tau}\right)$ and local model $f_i\left(\theta_i^{t\tau+r}\right)$, respectively:
\begin{align}\label{eq:equation25}
	f^{\rm lin} \left(\theta^{t\tau}\right)& = f\left(\theta^0\right) + J\left(\theta^0\right)^T\left(\theta^{t\tau}-\theta^0\right) \\
	f_i^{\rm lin} \left(\theta_{i}^{t\tau+r}\right)& = f_i\left(\theta^0\right) + J_i\left(\theta^0\right)^T\left(\theta_i^{t\tau+r}-\theta^0\right)
\end{align}
Our main results are as follows.
\begin{theorem}\label{theorem2}
	Under Assumptions~\ref{ass:ntkfullrank} to \ref{ass:activation}, when $n\to\infty$, we have
	\begin{align}
			& \sup_{t\geq 0}\left\|f^{\rm lin}(\theta^{t\tau})-f(\theta^{t\tau})\right\|_2 = \mathcal{O}\big(n^{-\frac{1}{2}}\big), \\
			& \!\!\sup_{\substack{t\geq 0,\\1\leq  r \leq \tau}}\left\|f_i^{\rm lin}\left(\theta_i^{t\tau+r}\right)-f_i\left(\theta_i^{t\tau+r}\right)\right\|_2 = \mathcal{O}\big(n^{-\frac{1}{2}}\big), \forall i \!\!
	\end{align}
\end{theorem}
The detailed proof is provided in appendix~\ref{appendixB}.

\begin{remark}[\bf Infinite-width FedAvg induces linearized global/local models ]
Theorem \ref{theorem2} suggests that as the network width approaches infinity, the global and local models become linear models. This extends the findings of \citet{shi2021fed}, which demonstrated that the global model can be approximated by a linear model, to show that both local and global models can be well approximated by linear models.
\end{remark}
Therefore, we can analyze the training dynamic of those linear models instead. The main theorem describing their training dynamics is presented as follows.
\begin{theorem}\label{theorem3}
	Under Assumptions ~\ref{ass:ntkfullrank} to \ref{ass:activation},  when $n\to \infty$ and $\eta_0\tau$ is sufficiently small such that the terms of $\mathcal{O}\left(\eta_0^2\tau^2\right)$ and higher are neglected, the linear model has closed-form expressions for the global parameters and outputs throughout the training process:
	\begin{align}
		&\theta^{t\tau}\! =\! -\frac{1}{n}J(\theta^0)(\Theta^0)^{-1}\!\!\left(I-e^{-\frac{\eta_0t\tau}{|\mathcal{D}|}\Theta^0}\right)g(\theta^0)+\theta^0, \label{eqn:theta}\\
		&f^{\rm lin}(x, \theta^{t\tau})=  f(x, \theta^0)\nonumber\\
		& \quad\quad\quad \quad - \Theta^0(x)(\Theta^0)^{-1}\left(I-e^{-\frac{\eta_0t\tau}{|\mathcal{D}|}\Theta^0}\right)g(\theta^0), \label{eqn:f}
	\end{align}
	where $\Theta^0(x) \triangleq \frac{1}{n}J(x, \theta^0)^TJ(\theta^0)$.
\end{theorem} 
The detailed proof is provided in appendix~\ref{appendixD}.

Suppose there is a model having the same structure that trains on the global dataset $\mathcal{D}$ via centralized GD, whose model parameters at the $t'$-th GD iteration is denoted by $\theta_{\rm cen}^{t'}$ and the model output is $f(x, \theta_{\rm cen}^{t'})$. When the initialization of $\theta_{\rm cen}$ and $\theta$ are the same, i.e., $\theta_{\rm cen}^{0} = \theta^{0}$, the following can be obtained according to  \citet[Equations (8), (10), (11)]{lee2019wide}:
	\begin{align}
	&\theta_{\rm cen}^{t'} = -\frac{1}{n}J(\theta^0)(\Theta^0)^{-1}\left(I-e^{-\frac{\eta_0\Theta^0 t'}{|\mathcal{D}|}}\right)g(\theta^0)+\theta^0,  \label{eqn:theta_c}\\
	&f_{\rm cen}\big(x, \theta_{\rm cen}^{t'}\big) = f\left(x, \theta^0\right) \nonumber\\
	& \quad\quad\quad\quad- \Theta^0(x)\big(\Theta^0\big)^{-1}\left(I-e^{-\frac{\eta_0\Theta^0t'}{|\mathcal{D}|}}\right)g(\theta^0). \label{eqn:f_c}
\end{align}
When $t' = t\tau$, by comparing \eqref{eqn:theta}, \eqref{eqn:f} with \eqref{eqn:theta_c}, \eqref{eqn:f_c} and employing Theorem \ref{theorem2},   we can obtain
\begin{align}
\theta_{\rm cen}^{t'} = \theta^{t\tau}, \quad f(x, \theta_{\rm cen}^{t'}) = f(x, \theta^{t\tau}) \label{eqn:cen}
\end{align}

\begin{remark}[\bf Infinite-width FedAvg generalizes the same as centralized GD]
Equations \eqref{eqn:cen} suggest that when the total number iterations of centralized GD and FedAvg are the same, both models share the same model parameters in the infinite-width regime, thereby producing the same output for an arbitrary test input and  achieving the same generalization performance. This means that the impacts of data heterogeneity on the generalization performance vanishes. 
\end{remark}

\section{Numerical Experiments}
In this section, we verify our theoretical findings by numerical experiments spanning various network architectures, loss functions, and optimization methods. Specifically, we evaluate the impact of data heterogeneity under different network widths, verify that both the local and global models of overparameterized FedAvg can be well approximated by linear models, and demonstrate that overparameterized FedAvg generalizes the same as centralized learning. The number of clients in our experiments are set to $M=10$, and the dataset as well as the model settings are provided as follows.\footnote{Codes to reproduce the main results are available at \url{https://github.com/kkhuge/ICML2025}.}

\begin{figure*}[tb]
	\centering
	\subfigure{
		\includegraphics[width=0.31\textwidth]{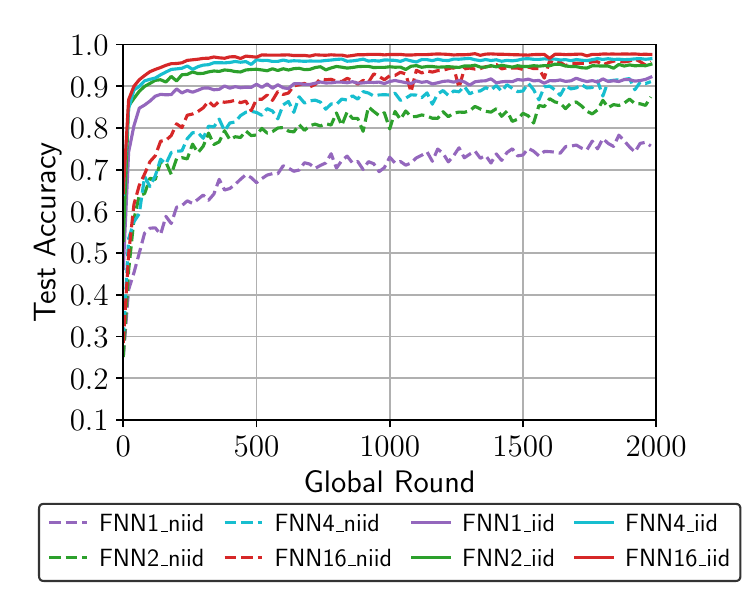}
		\label{fig:data_heterogeneity_mnist_fully_connected}
	}
	\subfigure{
		\includegraphics[width=0.31\textwidth]{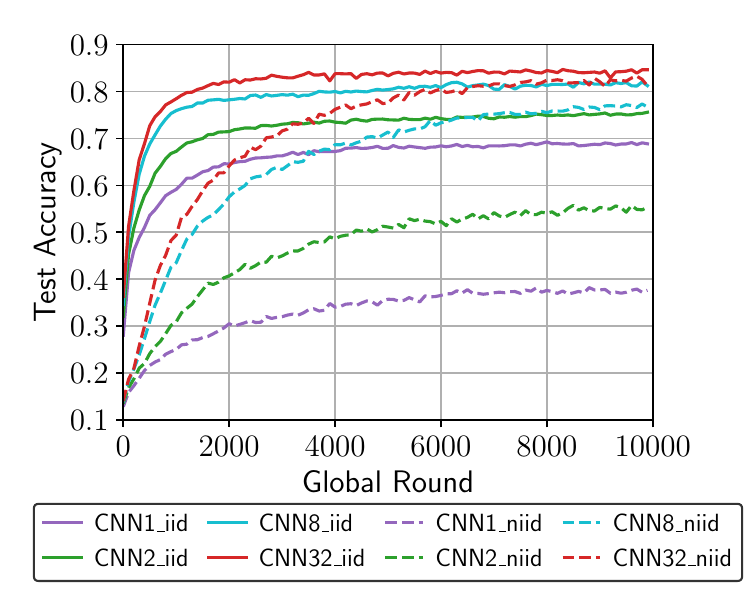}
		\label{fig:test_accuracy_CNN}
	}
	\subfigure{
		\includegraphics[width=0.31\textwidth]{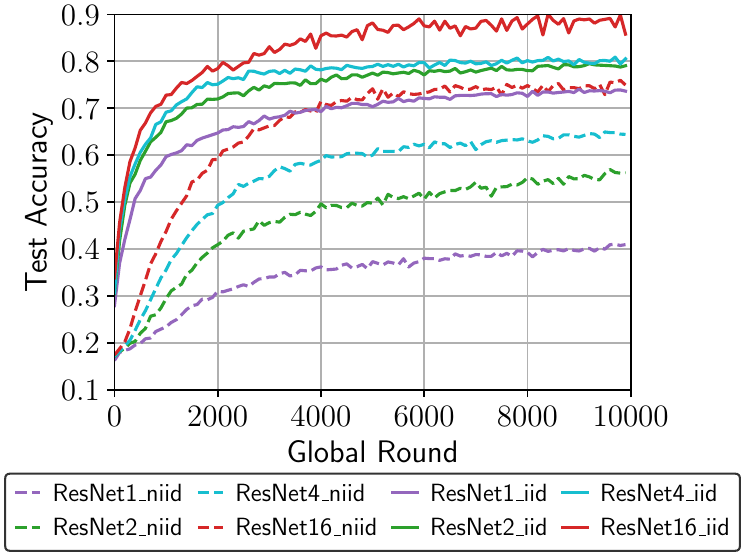}
		\label{fig:test_accuracy_ResNet}
	}
	\caption{Test accuracy of different network families. Each global round consists of $\tau = 5$ local SGD iterations. The left figure shows the test accuracy of FNNs on both IID and non-IID MNIST datasets. The middle and the right figures show the test accuracy of CNNs and ResNets, respectively, on both IID and non-IID CIFAR-10 datasets.}
	\label{fig:test_accuracy_networks}
\end{figure*}
\begin{figure*}[tb]
	\centering
	\subfigure{
		\includegraphics[width=0.31\textwidth]{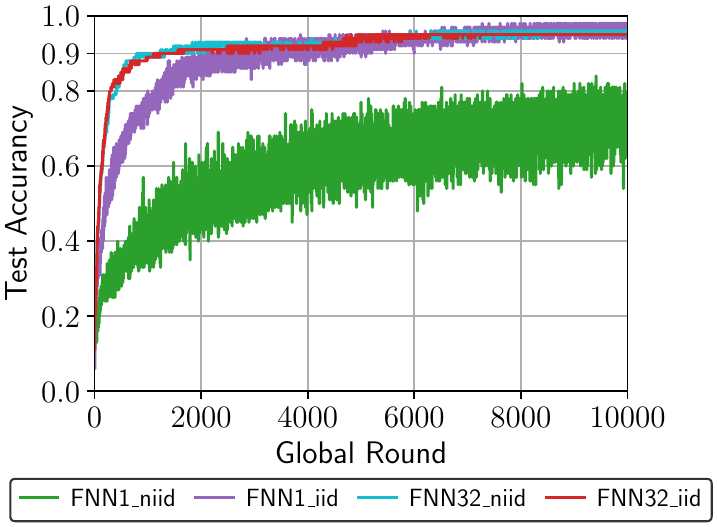}
		\label{fig:test acc fc large}
	}
	\hspace{5mm}
	\subfigure{
		\includegraphics[width=0.31\textwidth]{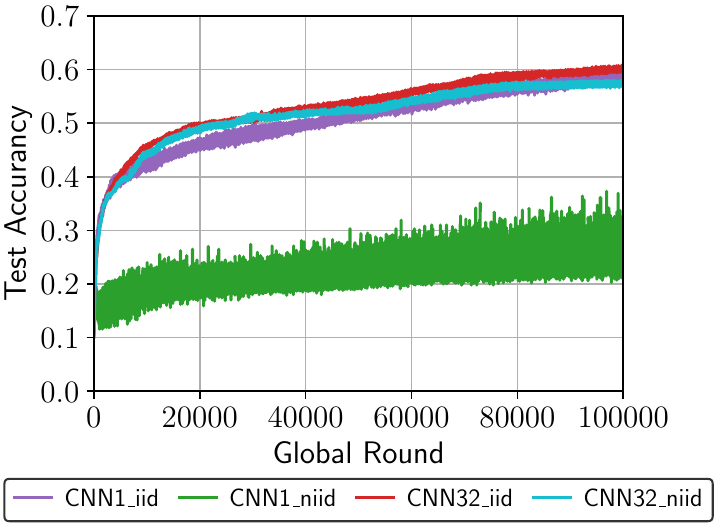}
		\label{fig:test acc cnn large}
	}
	\caption{Test accuracy of large networks. Each global round consists of $\tau = 5$ local SGD iterations, $\sigma_{W} = 1$, $\sigma_b = 0.1$. The left figure shows the test accuracy of the FNN$32$ and FNN$1$ on both IID and non-IID MNIST datasets. The right figure shows the test accuracy of the CNN$32$ and CNN$1$ on both IID and non-IID CIFAR-10 datasets.}
	\label{fig:Test accuracy of large networks}
\end{figure*}
\begin{figure*}[tb]
	\centering
	\subfigure{
		\includegraphics[width=0.31\textwidth]{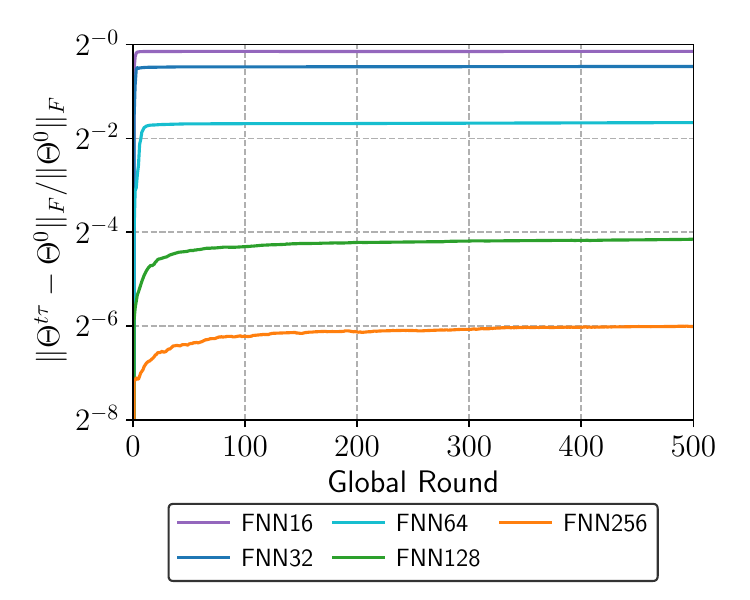}
		\label{fig: theta diff}
	}
	\subfigure{
		\includegraphics[width=0.31\textwidth]{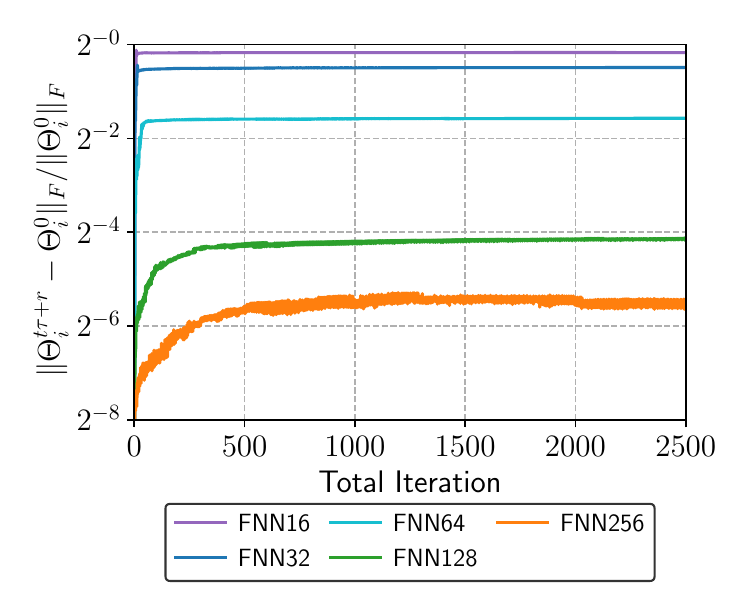}
		\label{fig: local theta diff}
	}
	\subfigure{
		\includegraphics[width=0.31\textwidth]{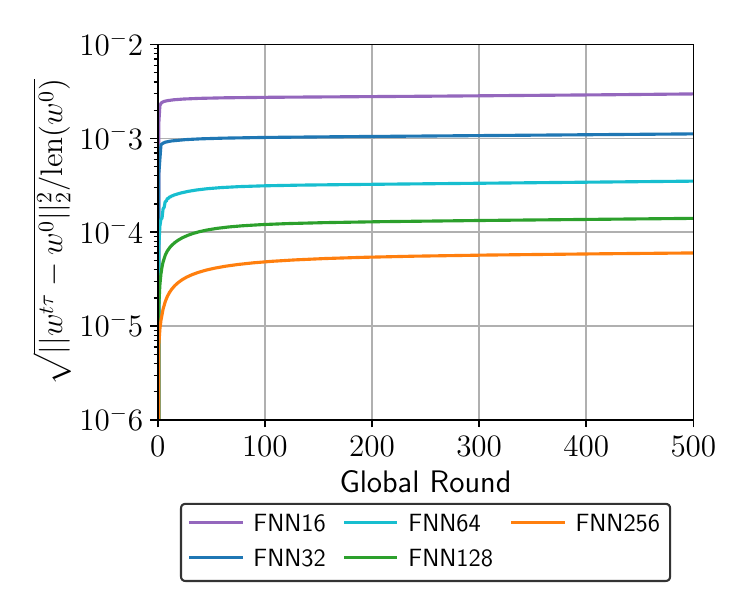}
		\label{fig:parameters change}
	}
	\caption{Training dynamic of NTK and model parameters. Each global round consists of $\tau=5$ local GD iterations, $\eta_0 = 1$, $\sigma_{W} = 1.5$, $\sigma_b = 0.1$.  The left figure shows the variation in the global NTK, the middle figure show the variation in a randomly chosen local NTK, while the right figure shows the model parameters' update during the training process.}
	\label{fig:training_dynamic}
\end{figure*}
\begin{figure}[ht]
	\centering
	\subfigure{
		\includegraphics[width=0.46\columnwidth]{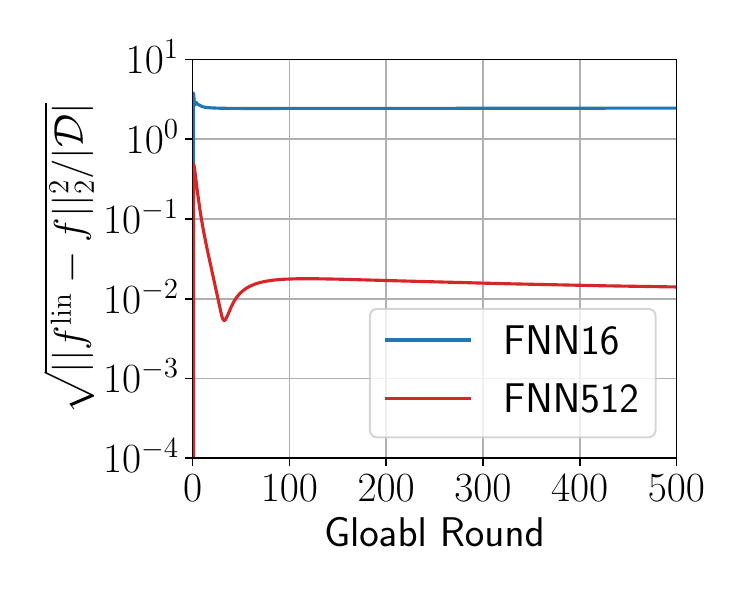}
		\label{fig:output diff fed and lin}
	}
	\subfigure{
		\includegraphics[width=0.46\columnwidth]{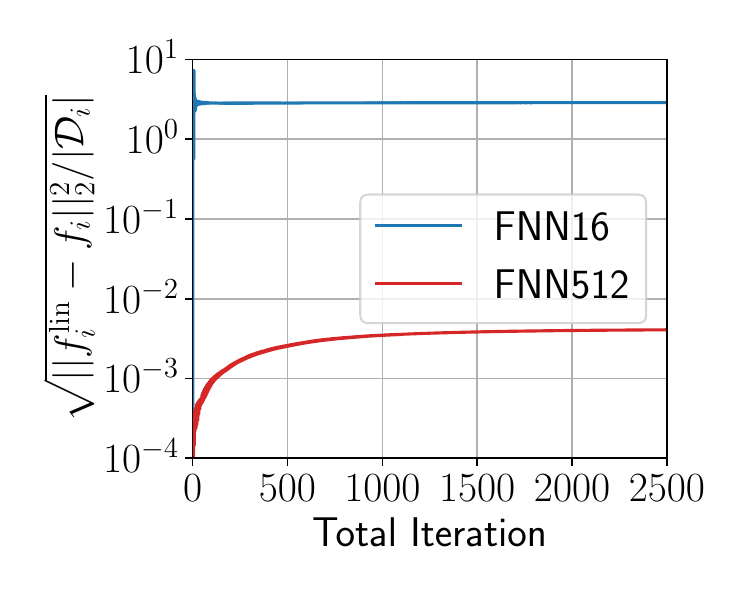}
		\label{fig:output diff fed and lin client 0}
	}
	\caption{Output difference between FedAvg and linear model. Each global round consists of $\tau=5$ local GD iterations, $\eta_0 = 1$, $\sigma_{W} = 1.5$, $\sigma_{b}=0.1$. }
	\label{fig:training_difference}
\end{figure}
\begin{figure}[ht]
	\centering
	\subfigure{
		\includegraphics[width=0.46\columnwidth]{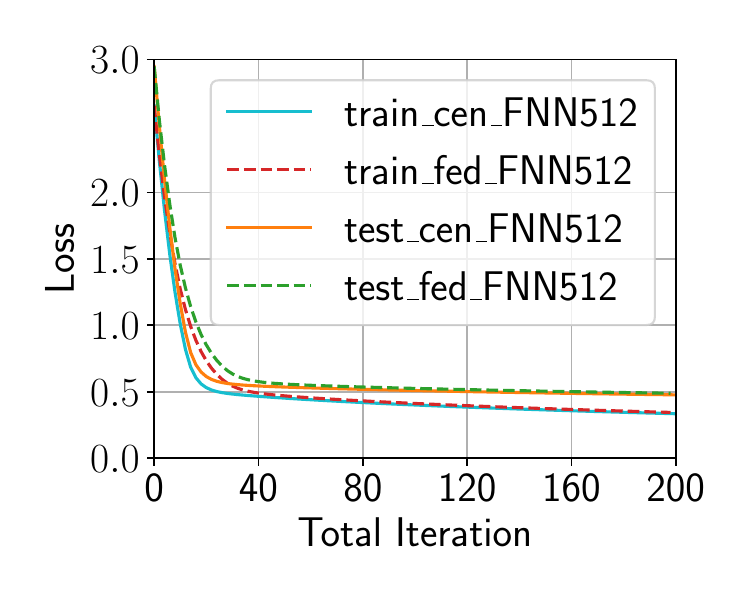}
		\label{fig:tau2}
	}
	\subfigure{
		\includegraphics[width=0.46\columnwidth]{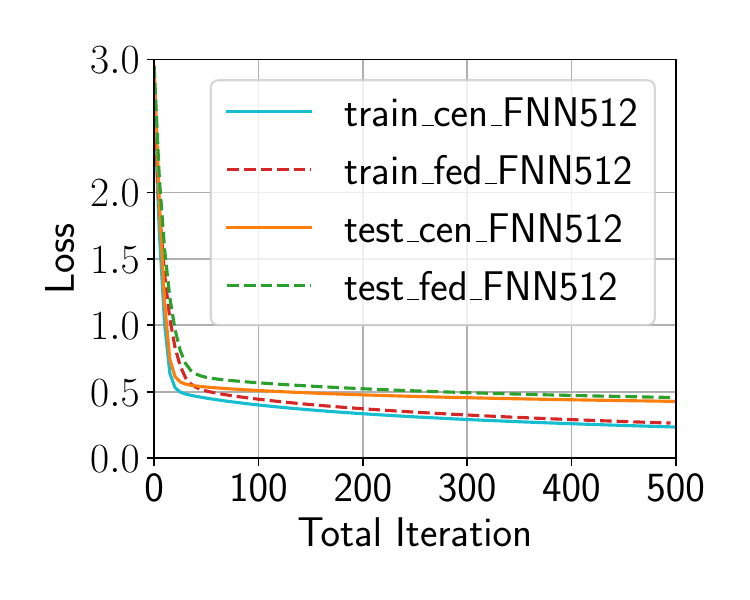}
		\label{fig:tau5}
	}
	\caption{Training and testing loss of FedAvg and centralized learning. $\eta_0 = 0.1$, $\sigma_{W} = 1.5$, $\sigma_{b}=0.1$. The global round of FedAvg consists of $\tau = 2$ and $\tau = 5$ local GD iterations in the left and right figures, respectively.}
	\label{fig:Difference between FedAvg with multiple local rounds and centralized learning}
\end{figure}

\paragraph{Non-IID Data Generation.} 1) Standard dataset: We conduct experiments on two widely used image classification datasets: MNIST \cite{lecun1998gradient} and CIFAR-10 \cite{cifar10}. To partition the datasets among different clients and generate non-IID data, we follow the approach proposed by \citet{hsu2019measuring}, which employs the Dirichlet distribution with a concentration parameter $\alpha$ controlling data heterogeneity. Specifically, a smaller value of $\alpha$ indicates a higher degree of data heterogeneity and we set $\alpha = 0.1$ throughout our experiments. 

2) Small dataset: To facilitate MSE loss minimization using gradient descent (GD) as required by our theoretical derivations, we also use the mini-MNIST and mini-CIFAR-10 datasets for binary image classification tasks. The mini-MNIST dataset is created by randomly selecting two classes from the MNIST dataset, followed by randomly sampling 50 images from each class for the training set and 10 images from each class for the test set. A similar approach is used to generate the mini-CIFAR-10 dataset, which contains 500 training images and 100 test images. To generate non-IID data, inspired by \citet{mcmahan2017communication, zhang2021fedpd}, we assign each class exclusively to specific clients: half of the clients receive all images from one class, while the remaining clients receive all images from the others.

\paragraph{Experimental Models.} We employ three types of models: FNN, convolutional neural networks (CNNs), and residual networks (ResNets). 

1) FNN: The structure of FNNs is detailed in Table~\ref{FC Networks} of Appendix~\ref{sec:network}, where a width factor $k$ is introduced to adjust the network width. By setting \(k = 1, 2, 4, 16\), we construct networks of varying widths, named FNN\(1\), FNN\(2\), FNN\(4\), and FNN\(16\).

2) CNN: We adopt the approach described by \citet{park2019effect} to obtain CNNs with varying widths having the base architecture in \citep[Figure 2]{lecun1998gradient}. The details of the architecture are presented in Table~\ref{Convolution Network} of Appendix~\ref{sec:network}, where the width factor \(k\) is used to scale the channel size. By setting \(k = 1, 2, 8, 32\), we generate CNN\(1\), CNN\(2\), CNN\(8\), and CNN\(32\). 

3) ResNet: The network architecture is based on the work of \citet{zagoruyko2016wide}, as shown in Table~\ref{Residual Network} of Appendix~\ref{sec:network}. The parameter \(\psi\) represents the number of blocks in each group, which is set to \(\psi = 1\) in our experiments. The channel size is fixed at 16, and we vary the width factor \(k = 1, 2, 4, 16\) to obtain ResNet\(1\), ResNet\(2\), ResNet\(4\), and ResNet\(16\).

\subsection{Impact of Non-IID Versus Network Width}
Although our theoretical analysis is based on GD with learning rate $\eta = \mathcal{O}(n^{-1})$, to show our conclusions can be extended to more practical settings, we use SGD with batch size $64$ and set a common learning rate $\eta=0.1$ with a weight decay of $0.0005$. Moreover, for MNIST and CIFAR-10, we use the practical cross-entropy loss  instead of the MSE loss required in the theoretical analysis. As shown in Figure~\ref{fig:test_accuracy_networks}, in the non-IID cases, the test accuracy of FNN\(1\), FNN\(2\), FNN\(4\), and FNN\(16\) decreases by \(17.4\%\), \(9.5\%\), \(6.3\%\), and \(2.0\%\), respectively, compared to the IID cases. Similarly, the test accuracy of CNN\(1\), CNN\(2\), CNN\(8\), and CNN\(32\) drops by \(44.9\%\), \(26.7\%\), \(5.1\%\), and \(2.4\%\), while the test accuracy of ResNet\(1\), ResNet\(2\), ResNet\(4\), and ResNet\(16\) decreases by \(44.6\%\), \(29.1\%\), \(18.7\%\), and \(14.8\%\), respectively. These results verify that the impact of data heterogeneity diminishes as the network width increases.

To further verify that the impact of data heterogeneity vanishes as the network width approaches infinity, we set the learning rate with $\eta = \frac{\eta_0}{n}$ in line with our theoretical analysis. As shown in Figure~\ref{fig:Test accuracy of large networks}, the convergence rate and final accuracy of FNN$32$ are nearly identical for both IID and non-IID data, and a similar trend is observed for CNN$32$. In contrast, a noticeable gap exists in FNN$1$ between IID and non-IID data, which is also evident in CNN$1$.

Additionally, to monitor the evolution of model parameters and the global/local NTKs during training, we train FNNs on the non-IID mini-MNIST dataset using GD for binary classification with MSE loss. As shown in Figure~\ref{fig:training_dynamic},  increasing the network width diminishes the variations in both the global/local NTKs and the model parameters. For sufficiently wide networks,  the global and local NTKs as well as the model parameters remain nearly constant, exhibiting a lazy training behavior.

\subsection{Linear Approximation of FedAvg}\label{experiment2}
To show that overparameterized FedAvg can be well approximated by linear models, we train FNN$16$ and FNN$512$ on the non-IID mini-MNIST dataset with MSE loss for binary image classification using GD. In Figure~\ref{fig:training_difference}, we analyze the difference in the outputs between the global model $f$ in FedAvg and the global linear model $f^{\rm lin}$ throughout the training process. Additionally, a randomly selected local model is examined by comparing its outputs with those of the corresponding linear model $f_i^{\rm lin}$.

As shown in Figure~\ref{fig:training_difference}, we can observe that, for FNN$512$, the outputs of the global model $f$ and the linear model $f^{\rm lin}$ remain nearly identical throughout  training. By contrast, the narrower FNN$16$ exhibits noticeable difference between $f$ and $f^{\rm lin}$. A similar trend is observed for the local models $f_i$ and $f_i^{\rm lin}$. These findings confirm Theorem 2 that wider networks enable linear approximations to align more closely with the dynamics of FedAvg.

\subsection{Comparison of FedAvg with Centralized Learning}\label{section6.3}
To compare overparameterized FedAvg with centralized learning, in Figure~\ref{fig:Difference between FedAvg with multiple local rounds and centralized learning} we evaluate the loss of $f_{\rm cen}(x, \theta^{t'})$ and  $f(x, \theta^{t\tau})$ on both the training and testing dataset of mini-CIFAR-10, by ensuring $t'=t\tau$ for a fair comparison. It can be observed that the outputs of FedAvg and centralized learning are almost identical under the same number of GD iterations,  empirically confirming Theorem \ref{theorem3} that overparameterized FedAvg generalized the same as centralized learning.


\section{Conclusion and Future Directions}
In this work, we established a quantitative relationship between neural network width and the impact of data heterogeneity in FedAvg.  We  proved that the impact of data heterogeneity on the convergence of FedAvg diminishes at a rate of $\mathcal{O}(n^{-\frac{1}{2}})$ with increasing network width $n$ and vanishes entirely in the infinite-width limit. In that regime, we extended NTK theory from centralized learning to FL, showing that both the global and local models in FedAvg are linear and have constant NTKs. Furthermore, we derived closed-form expressions for the model outputs of FedAvg, revealing the equivalence between infinite-width FedAvg and centralized GD in both training dynamics and generalization performance, under matched training iterations. Extensive experiments on MINST and CIFAR-10 datasets validated our conclusions across different network architectures, loss functions, and optimization methods.

These theoretical findings provide valuable insights for practical federated learning. Notably, the linear dependence between model outputs and parameters suggests a potential communication-efficient FL strategy: clients may transmit only the model outputs instead of the complete model parameters for aggregation. This approach may significantly reduce the communication overhead of FL and deserves further investigation. Another essential  direction for future research involves extending these analyses to more realistic FL settings by relaxing the idealized assumptions, such as infinite network width and continuous-time gradient flow.

\section*{Acknowledgments}

We sincerely thank the anonymous reviewers for their valuable comments and suggestions. This work is supported by National Natural Science
Foundation of China under Grant 62301015.
\section*{Impact Statement}

This paper presents work whose goal is to advance the field of 
Machine Learning. There are many potential societal consequences 
of our work, none which we feel must be specifically highlighted here.


\bibliography{icml25}
\bibliographystyle{icml2025}

\newpage
\appendix
\onecolumn
\section{Essential Lemmas}\label{lemmas}
In this section, we introduce the necessary lemmas used in deriving Theorems \ref{theorem1} $\sim$ \ref{theorem3}. 

\begin{lemma}[\bf Local Lipschitzness of the Jacobian]\label{lem:lemma1}
There exists a constant $C>0$, such that for any $C'>0$, with high probability over random initialization the following holds:
	\begin{align}
		\left\{
		\begin{aligned}
			&\frac{1}{\sqrt{n}}\left\|J(\theta)-J(\theta')\right\|_F\leq C \left\|\theta-\theta'\right\|_2, \\
			& \frac{1}{\sqrt{n}}\left\|J(\theta)\right\|_F\leq C,
		\end{aligned}
		\right. \quad 
		\forall \theta,\theta'\in B(\theta_0,C'n^{-\frac{1}{2}})
	\end{align}
	where $B(\theta_0,R)\triangleq\{\theta:\left\|\theta-\theta_0\right\|_2<R\}$.
\end{lemma}
Lemma~\ref{lem:lemma1} has been proved by \citet[Lemma 1]{lee2019wide} and we will apply it directly.

\begin{lemma}[\bf Local Lipschitzness of the Local Jacobian]\label{lem:lemma2}
	There exists a constant $C>0$, such that for any $C'>0$, with high probability over random initialization the following holds:
	\begin{align}
		\left\{
		\begin{aligned}
			& \frac{1}{\sqrt{n}}\left\|J_i(\theta)-J_i(\theta')\right\|_F\leq C\left\|\theta-\theta'\right\|_2, \\
			& \frac{1}{\sqrt{n}}\left\|J_i(\theta)\right\|_F\leq C,
		\end{aligned}
		\right.
		\quad \forall \theta,\theta'\in B(\theta_0,C'n^{-\frac{1}{2}})
	\end{align}
	where $B(\theta_0,R)\triangleq \{\theta:\left\|\theta-\theta_0\right\|_2<R\}$.
\end{lemma}
\begin{proof}
	Since $J(\theta)-J(\theta')$ is the concatenation of all $J_i(\theta)-J_i(\theta')$ for $i=1,\cdots, M$, we have
	\begin{equation}
		\frac{1}{\sqrt{n}}||J_i(\theta)-J_i(\theta')||_F \leq\frac{1}{\sqrt{n}}||J(\theta)-J(\theta')||_F \leq C||\theta-\theta'||_2,
	\end{equation}
	where the last step applies Lemma~\ref{lem:lemma1}. Similarly, since $J(\theta)$ is the concatenation of all $J_i(\theta)$, we have
	\begin{equation}
		\frac{1}{\sqrt{n}}||J_i(\theta)||_F\leq \frac{1}{\sqrt{n}}||J(\theta)||_F \leq C.
	\end{equation}

\end{proof}
\begin{lemma}\label{lem:lemma3}
	For a square matrix A whose norm satisfies $||A||\leq \rho_A$,  the remainder term $\Omega(e^{-A}) \triangleq \sum_{k=2}^{\infty}(-1)^k\frac{A^k}{k!}$, i.e., the sum of second-order and higher terms in the Taylor series expansion of $e^{-A}$, satisfies
	\begin{equation}
	\Omega\left(e^{-A}\right)	\leq \frac{\rho_A^2}{2}e^{\rho_A}.
	\end{equation}
\end{lemma}
\begin{proof}
By taking the norm on both sides of $\Omega(e^{-A}) = \sum_{k=2}^{\infty}(-1)^k\frac{A^k}{k!}$, we can readily obtain
 	\begin{equation}
		\left\|\Omega\left(e^{-A}\right)\right\|
		\leq \sum_{k=2}^{\infty}\frac{\left\|A\right\|^k}{k!}\leq\frac{\left\|A\right\|^2}{2}\sum_{k=2}^{\infty}\frac{\left\|A\right\|^{k-2}}{(k-2)!}\leq \frac{\rho_A^2}{2}e^{\left\|A\right\|}\leq \frac{\rho_A^2}{2}e^{\rho_A}.
	\end{equation}
\end{proof}

\begin{lemma}\label{lem:lemma4}
	For any non-negative aggregation weights combination $\{p_i\}_{i=1,\cdots, M}$, if $\eta_0\tau$ is sufficiently small such that the terms of $\mathcal{O}(\eta_0^2\tau^2)$ and higher are neglected,  the following equation holds:
	\begin{equation}
		\sum_{i=1}^{M}p_ie^{-\eta_0\tau\Theta_i} = e^{-\eta_0\tau\sum_{i=1}^{M}p_i\Theta_i}. \label{eqn:lemma4}
	\end{equation} 
\end{lemma}
\begin{proof}
Employing the Taylor series expansion on both sides of \eqref{eqn:lemma4} yields
	\begin{equation}
		\sum_{i=1}^{M}p_ie^{-\eta_0\tau\Theta_i} = \sum_{i=1}^{M}p_i\left[I-\eta_0\tau\Theta_i+\mathcal{O}\left(\eta_0^2\tau^2\right)\right] = I-\eta_0\tau\sum_{i=1}^{M}p_i\Theta_i , \label{eqn:left}
	\end{equation}
and
	\begin{equation}
		e^{-\eta_0\tau\sum_{i=1}^{M}p_i\Theta_i} = I-\eta_0\tau\sum_{i=1}^{M}p_i\Theta_i+\mathcal{O}\left(\eta_0^2\tau^2\right) = I-\eta_0\tau\sum_{i=1}^{M}p_i\Theta_i, \label{eqn:right}
	\end{equation}
respectively. Comparing \eqref{eqn:left} with \eqref{eqn:right}, we can obtain
	\begin{equation}
		\sum_{i=1}^{M}p_ie^{-\eta_0\tau\Theta_i} = e^{-\eta_0\tau\sum_{i=1}^{M}p_i\Theta_i}.
	\end{equation}
\end{proof}

\section{Proof of Theorem~\ref{theorem1}}\label{appendixA}
Since the width of network is large, with the initialization described in Section~\ref{section3}, the output $f(\theta^0)$ converges to Gaussian process $\mathcal{N}(0,\mathcal{K}(\mathcal{X},\mathcal{X}))$ according to the central limit theorem, where  $\mathcal{K}(\mathcal{X},\mathcal{X})=\lim\limits_{n\to\infty}\mathbb{E}\left[f(\theta^0)f(\theta^0)^T\right]$. Therefore, for arbitrarily small $\delta_0>0$, there exist constants $R_0>0$ and $n'$ such that for any $n\geq n'$, with probability at least $1-\delta_0$ over random initialization, we have
\begin{equation}
	\left\|g\left(\theta^0\right)\right\|_2\leq R_0. \label{eqn:n1}
\end{equation}
In the following, we prove Theorem~\ref{theorem1} by mathematical induction, where the induction hypotheses are 
\begin{align}
	\left\|g(\theta^{t\tau})\right\|_2 & \leq q^{t}R_0+\frac{2\eta_0\tau CC_1R_0\zeta\left(1-q^{t}\right)}{\left(1-q\right)^2},\label{equation39} \\
	\left\|\theta^{t\tau}-\theta^0\right\|_2 & \leq \frac{\eta_0 \tau C R_0\left(1-q^t\right)}{\sqrt{n}\left(1-q\right)}. \label{equation40}
\end{align}
When $t=0$, \eqref{equation39} and \eqref{equation40} trivially hold. Then, we aim to prove the hypotheses hold in the $(t+1)$-th global round, i.e., 
\begin{align}
	\left\|g(\theta^{(t+1)\tau})\right\|_2 & \leq q^{t+1}R_0+\frac{2\eta_0\tau CC_1R_0\zeta\left(1-q^{t+1}\right)}{\left(1-q\right)^2}, \label{eq:equation64} \\ 
	\left\|\theta^{(t+1)\tau}-\theta^0\right\|_2 & \leq \frac{\eta_0 \tau C R_0\left(1-q^{t+1}\right)}{\sqrt{n}\left(1-q\right)}. \label{eq:equation65}
\end{align}
Referring to \eqref{eq:equation13}, the local update step of client $i$ in the $(t\tau+r+1)$-th iteration is
\begin{equation}
\theta_i^{t\tau+r+1} = \theta_i^{t\tau+r}-\frac{\eta}{|\mathcal{D}_i|} J_i\left(\theta_i^{t\tau+r}\right)g_i\left(\theta_i^{t\tau+r}\right).
\end{equation}
Since $\eta = \frac{\eta_0}{n}$ is small, we can approximate the local update with gradient flow by making time continuous, yielding
\begin{equation}
	\frac{d\theta_i^{t\tau+r}}{dr}=-\frac{\eta}{|\mathcal{D}_i|} J_i\left(\theta_i^{t\tau+r}\right) g_i\left(\theta_i^{t\tau+r}\right) .
\end{equation}
Further applying the chain rule, we can obtain
\begin{align}\label{eq:equation53}
	\frac{dg\left(\theta_i^{t\tau+r}\right)}{dr}
	&=  J\left(\theta_i^{t\tau+r}\right)^T  \frac{d\theta_i^{t\tau+r}}{dr}\nonumber\\
	& = -\frac{\eta}{|\mathcal{D}_i|} J\left(\theta_i^{t\tau+r}\right)^TJ_i\left(\theta_i^{t\tau+r}\right) g_i\left(\theta_i^{t\tau+r}\right)\nonumber\\
	& = -\frac{\eta_0}{n|\mathcal{D}_i|} J\left(\theta_i^{t\tau+r}\right)^TJ_i\left(\theta_i^{t\tau+r}\right) P_ig\left(\theta_i^{t\tau+r}\right)\nonumber\\
	&= -\frac{\eta_0\Theta_i^{t\tau+r}}{|\mathcal{D}_i|} g\left(\theta_i^{t\tau+r}\right),
\end{align}
where the last step follows from the definition of the local NTK in \eqref{equation21}. By integrating both sides of \eqref{eq:equation53} from $0$ to $r$ and applying the mean value theorem for integrals, we have
\begin{equation}\label{eq:equation49}
	g\left(\theta_i^{t\tau+r}\right) = e^{-\frac{\eta_0\tau}{|\mathcal{D}_i|}\Theta_i^{t\tau+\hat{r}_i}} g\left(\theta_i^{t\tau}\right),~\hat{r}_i\in\left(0,r\right)
\end{equation}
Replacing $r$ with $\tau$ yields
\begin{equation}
	g\left(\theta_i^{t\tau+\tau}\right) = e^{-\frac{\eta_0\tau}{|\mathcal{D}_i|}\Theta_i^{t\tau+\bar{r}_i}} g\left(\theta_i^{t\tau}\right),~\bar{r}_i\in\left(0,\tau\right) \label{eqn:gtheta}
\end{equation}

\subsection{Proof of \eqref{eq:equation65}}
We first prove \eqref{eq:equation65}. To bound $\big\|\theta^{(t+1)\tau}-\theta^0\big\|_2$, we proceed the following derivations.
\begin{align}
	\frac{d\left\|\theta_i^{t\tau+r}-\theta_i^0\right\|_2}{dr}
	&\leq \left\|\frac{d\theta_i^{t\tau+r}}{dr}\right\|_2 \nonumber\\
	& = \frac{\eta}{|\mathcal{D}_i|}\left\|J_i\left(\theta_i^{t\tau+r}\right)g_i\left(\theta_i^{t\tau+r}\right)\right\|_{2} \nonumber\\
	& = \frac{\eta}{|\mathcal{D}_i|}\left\|J_i\left(\theta_i^{t\tau+r}\right)P_ig\left(\theta_i^{t\tau+r}\right)\right\|_{2} \nonumber\\
	& \leq \frac{\eta}{|\mathcal{D}_i|}\left\|J_i\left(\theta_i^{t\tau+r}\right)\right\|_{F}\left\|P_i\right\|_{\rm op}\left\|g\left(\theta_i^{t\tau+r}\right)\right\|_{2}\nonumber\\
	&\overset{\sf (a)}{\leq} \frac{\eta C\sqrt{n}}{|\mathcal{D}_i|}\left\|g\left(\theta_i^{t\tau+r}\right)\right\|_{2}\nonumber\\
	&= \frac{\eta_0 C}{\sqrt{n}|\mathcal{D}_i|}\left\|e^{-\frac{\eta_0\tau}{|\mathcal{D}_i|}\Theta_i^{t\tau+\hat{r}_i}} g\left(\theta^{t\tau}\right)\right\|_{2}\nonumber\\
	&\leq \frac{\eta_0 C}{\sqrt{n}|\mathcal{D}_i|}\left\|e^{-\frac{\eta_0\tau}{|\mathcal{D}_i|}\Theta_i^{t\tau+\hat{r}_i}}\right\|_{\rm op} \left\|g\left(\theta^{t\tau}\right)\right\|_{2}\nonumber\\
	&\leq \frac{\eta_0 C}{\sqrt{n}|\mathcal{D}_i|} \left\|g\left(\theta^{t\tau}\right)\right\|_{2}, \label{eqn:dtheta}
\end{align}
where the first step is obtained by applying the chain rule and Cauchy-Schwarz inequality, step $\sf (a)$ holds because of Lemma~\ref{lem:lemma2} and $\left\|P_i\right\|_{\rm op} = 1$, and the last step holds because $\Theta_i^{t\tau+\hat{r}_i}$ is not full rank  from its definition \eqref{equation21}, yielding $\big\|e^{-\frac{\eta_0\tau}{|\mathcal{D}_i|}\Theta_i^{t\tau+\hat{r}_i}}\big\|_{\rm op} \leq e^{-\frac{\eta_0\tau}{|\mathcal{D}_i|}\lambda_{\min}(\Theta_i^{t\tau+\hat{r}_i})} = 1$. Integrating from $0$ to $r$ on both sides of \eqref{eqn:dtheta} yields
\begin{equation}\label{eq:equation100}
	\left\|\theta_i^{t\tau+r}-\theta_i^0\right\|_2 \leq \frac{\eta_0 rC }{\sqrt{n}|\mathcal{D}_i|}\left\|g\left(\theta^{t\tau}\right)\right\|_{2} + \left\|\theta^{t\tau}-\theta^0\right\|_2.
\end{equation}
Further replacing $r$ with $\tau$ yields
\begin{equation}\label{eq:equation62}
	\left\|\theta_i^{t\tau+\tau}-\theta_i^0\right\|_2 \leq \frac{\eta_0 \tau C }{\sqrt{n}|\mathcal{D}_i|}\left\|g\left(\theta^{t\tau}\right)\right\|_{2} + \left\|\theta^{t\tau}-\theta^0\right\|_2.
\end{equation}
According to \eqref{eq:equation14}, we can obtain
\begin{align}
	\left\|\theta^{\left(t+1\right)\tau}-\theta^0\right\|_2 
	&= \left\|\sum_{i=1}^{M}p_i\left(\theta_i^{t\tau+\tau}-\theta^0\right)\right\|_2\nonumber\\
	&\leq \sum_{i=1}^{M}p_i\left\|\theta_i^{t\tau+\tau}-\theta^0\right\|_2\nonumber\\
	&\overset{\sf (a)}{\leq} \sum_{i=1}^{M}p_i\left[\frac{\eta_0 \tau C }{\sqrt{n}|\mathcal{D}_i|}\left\|g\left(\theta^{t\tau}\right)\right\|_{2} + \left\|\theta^{t\tau}-\theta^0\right\|_2\right]\nonumber\\
	&= \frac{M\eta_0 \tau C }{\sqrt{n}|\mathcal{D}|}\left\|g\left(\theta^{t\tau}\right)\right\|_{2} + \left\|\theta^{t\tau}-\theta^0\right\|_2\nonumber\\ 
	&\overset{\sf (b)}{\leq} \frac{\eta_0 \tau C }{\sqrt{n}}\left\|g\left(\theta^{t\tau}\right)\right\|_{2} + \left\|\theta^{t\tau}-\theta^0\right\|_2\nonumber\\
	&\overset{\sf (c)}{\leq} \frac{\eta_0 \tau C}{\sqrt{n}} \left[q^{t}\left(R_0-\frac{2\eta_0\tau CC_1R_0\zeta}{\left(1-q\right)^2}\right) + \frac{2\eta_0\tau CC_1R_0\zeta}{\left(1-q\right)^2}\right] + \frac{\eta_0 \tau C R_0\left(1-q^t\right)}{\sqrt{n}\left(1-q\right)}\nonumber\\
	&= \frac{\eta_0 \tau C R_0\left(1-q^{t+1}\right)}{\sqrt{n}\left(1-q\right)},
\end{align}
where step $\sf (a)$ comes from \eqref{eq:equation62}, step $\sf (b)$ holds because $|\mathcal{D}|\geq M$, step $\sf (c)$ applies the induction hypothesis \eqref{equation39} and \eqref{equation40}, and the last step omits $\frac{\zeta}{\sqrt{n}}=\mathcal{O}(n^{-1})$  according to Assumption~\ref{ass:wide}. Therefore, \eqref{eq:equation65} is proved. 

\subsection{Bounding the Model Divergence}
To prove \eqref{eq:equation64}, we first bound $\big\|\Delta \theta_i^{\left(t+1\right)\tau}\big\|_2$. According to \eqref{eq:equation100}, we have:
\begin{align}\label{eq:equation56}
	\left\|\theta_i^{t\tau+r}-\theta_i^0\right\|_2
	&\leq \frac{\eta_0 rC }{\sqrt{n}|\mathcal{D}_i|}\left\|g\left(\theta^{t\tau}\right)\right\|_{2} + \left\|\theta^{t\tau}-\theta^0\right\|_2\nonumber\\
	&\leq \frac{\eta_0 rC}{\sqrt{n}|\mathcal{D}_i|} \left[q^{t}\left(R_0-\frac{2\eta_0\tau CC_1R_0\zeta}{\left(1-q\right)^2}\right) + \frac{2\eta_0\tau CC_1R_0\zeta}{\left(1-q\right)^2}\right]+ \frac{\eta_0 \tau C R_0\left(1-q^t\right)}{\sqrt{n}\left(1-q\right)}\nonumber\\
	&=\frac{\eta_0 rC q^{t}R_0}{\sqrt{n}|\mathcal{D}_i|}+ \frac{\eta_0 \tau C R_0\left(1-q^t\right)}{\sqrt{n}\left(1-q\right)},
\end{align}
where the second step employs the induction hypotheses \eqref{equation39} and \eqref{equation40}, and the last step omits $\frac{\zeta}{\sqrt{n}}=\mathcal{O}(n^{-1})$. Referring to \eqref{eqn:di}, we have
\begin{align}\label{eq:equation39}
	\left\|\Delta \theta_i^{\left(t+1\right)\tau}\right\|_2 
	&= \left\|\theta_i^{t\tau+\tau} - \theta^{\left(t+1\right)\tau}\right\|_2\nonumber\\
	& = \left\|\left(\theta_i^{t\tau+\tau}-\theta^0\right)-\left(\theta^{\left(t+1\right)\tau}-\theta^0\right)\right\|_2 \nonumber\\
	& = \left\|\left(\theta_i^{t\tau+\tau}-\theta^0\right)- \sum_{i=1}^{M}p_i\left(\theta_i^{t\tau+\tau}-\theta^0\right)\right\|_2\nonumber\\
	& \leq \left\|\left(\theta_i^{t\tau+\tau}-\theta^0\right)\right\|_2 + \sum_{i=1}^{M}p_i\left\|\left(\theta_i^{t\tau+\tau}-\theta^0\right)\right\|_2\nonumber\\
	& \overset{\sf(a)}{\leq} \frac{\eta_0 \tau C q^{t}R_0}{\sqrt{n}|\mathcal{D}_i|}+ \frac{\eta_0 \tau C R_0\left(1-q^t\right)}{\sqrt{n}\left(1-q\right)} +\sum_{i=1}^{M}p_i\left[\frac{\eta_0 \tau C q^{t}R_0}{\sqrt{n}|\mathcal{D}_i|}+ \frac{\eta_0 \tau C R_0\left(1-q^t\right)}{\sqrt{n}\left(1-q\right)}\right]\nonumber\\
	& = \frac{\eta_0 \tau C q^{t}R_0}{\sqrt{n}|\mathcal{D}_i|}+ \frac{\eta_0 \tau C R_0\left(1-q^t\right)}{\sqrt{n}\left(1-q\right)} +\frac{M\eta_0 \tau C q^{t}R_0}{\sqrt{n}|\mathcal{D}|}+ \frac{\eta_0 \tau C R_0\left(1-q^t\right)}{\sqrt{n}\left(1-q\right)}\nonumber\\
	&\overset{\sf(b)}{\leq} \frac{\eta_0 \tau C q^{t}R_0}{\sqrt{n}}+ \frac{\eta_0 \tau C R_0\left(1-q^t\right)}{\sqrt{n}\left(1-q\right)} +\frac{\eta_0 \tau C q^{t}R_0}{\sqrt{n}}+ \frac{\eta_0 \tau C R_0\left(1-q^t\right)}{\sqrt{n}\left(1-q\right)}\nonumber\\
	&= \frac{2\eta_0\tau CR_0}{\sqrt{n}}\left(q^t+\frac{1-q^t}{1-q}\right)\nonumber\\
	&= \frac{2\eta_0\tau CR_0\left(1-q^{t+1}\right)}{\sqrt{n}\left(1-q\right)}\nonumber\\
	&\leq \frac{2\eta_0\tau CR_0}{\sqrt{n}\left(1-q\right)},
\end{align}
where the step $(a)$ is from  \eqref{eq:equation56}, step $(b)$ holds because $|\mathcal{D}_i|\geq 1$ and $|\mathcal{D}| \geq M$, and the last step hold because  we require $0\leq q<1$.\footnote{We prove that there exists $\eta_0>0$ such that $0\leq q<1$ at the end of Section~\ref{sec:B3}.} Therefore, the data heterogeneity term can be bounded by
\begin{equation}\label{eq:equation165}
	\sum_{i=1}^{M}p_i\left\|\Delta  \theta_i^{(t+1)\tau}\right\|_2 \leq \sum_{i=1}^{M}p_i \frac{2\eta_0\tau CR_0}{\sqrt{n}\left(1-q\right)} = \frac{2\eta_0\tau CR_0}{\sqrt{n}\left(1-q\right)} = \zeta.
\end{equation}
Notably, we have proven \eqref{eq:het}. 

\subsection{Proof of  \eqref{eq:equation64}} \label{sec:B3}
Finally, we prove \eqref{eq:equation64} to finish the induction. Taking the Taylor series expansion of  $g\left(\theta_i^{t\tau+\tau}\right)$  at $\theta^{\left(t+1\right)\tau}$ yields
\begin{equation}
	g\left(\theta_i^{t\tau+\tau}\right) = g\left(\theta^{\left(t+1\right)\tau}\right) + J\left(\theta^{\left(t+1\right)\tau}\right)^T\Delta \theta_i^{\left(t+1\right)\tau} + \Omega_i,
\end{equation}
where $\Omega_i$ represents the remainder terms of order two and above. Taking the sum of both sides yields
\begin{align}
	\sum_{i=1}^{M}p_ig\left(\theta_i^{t\tau+\tau}\right)
	&= \sum_{i=1}^{M}p_ig\left(\theta^{\left(t+1\right)\tau}\right) + \sum_{i=1}^{M}p_i J\left(\theta^{\left(t+1\right)\tau}\right)^T\Delta \theta_i^{\left(t+1\right)\tau} + \sum_{i=1}^{M}p_i\Omega_i \nonumber\\
	&=\sum_{i=1}^{M}p_ig\left(\theta^{\left(t+1\right)\tau}\right)+ \sum_{i=1}^{M}p_i\Omega_i \nonumber\\
	&=g\left(\theta^{\left(t+1\right)\tau}\right)+ \sum_{i=1}^{M}p_i\Omega_i, \label{eqn:sum}
\end{align}
where the second step holds because $\sum_{i=1}^{M}p_iJ\left(\theta^{\left(t+1\right)\tau}\right)^T\Delta \theta_i^{\left(t+1\right)\tau} = 0$ according to the model aggregation \eqref{eq:equation14}. Rewriting \eqref{eqn:sum}, we have
\begin{equation}\label{eq:equation20}
	g\left(\theta^{\left(t+1\right)\tau}\right)=\sum_{i=1}^{M}p_ig\left(\theta_i^{t\tau+\tau}\right)-\sum_{i=1}^{M}p_i\Omega_i.
\end{equation}
Taking the norm of both sides yields
\begin{equation}\label{eq:equation41}
	\left\|g\left(\theta^{\left(t+1\right)\tau}\right)\right\|_2\leq\left\|\sum_{i=1}^{M}p_ig\left(\theta_i^{t\tau+\tau}\right)\right\|_2+\left\|\sum_{i=1}^{M}p_i\Omega_i\right\|_2.
\end{equation}
In the following, we bound $\left\|\sum_{i=1}^{M}p_ig\left(\theta_i^{t\tau+\tau}\right)\right\|_2$ and $\left\|\sum_{i=1}^{M}p_i\Omega_i\right\|_2$, respectively. 

\paragraph{1) Bounding $\left\|\sum_{i=1}^{M}p_ig\left(\theta_i^{t\tau+\tau}\right)\right\|_2$:}  According to \eqref{eqn:gtheta}, we have
\begin{align}\label{eq:equation42}
	\left\|\sum_{i=1}^{M}p_ig\left(\theta_i^{t\tau+\tau}\right)\right\|_2
	&=\left\|\sum_{i=1}^{M}p_ie^{-\frac{\eta_0\tau}{|\mathcal{D}_i|}\Theta_i^{t\tau+\bar{r}_i}} g\left(\theta_i^{t\tau}\right)\right\|_2\nonumber\\
	&=\left\|\sum_{i=1}^{M}p_i\left[I-\frac{\eta_0\tau}{|\mathcal{D}_i|}\Theta_i^{t\tau+\bar{r}_i} +\Omega\left(e^{-\frac{\eta_0\tau}{|\mathcal{D}_i|}\Theta_i^{t\tau+\bar{r}_i}}\right)\right]g\left(\theta_i^{t\tau}\right)\right\|_2\nonumber\\
	&\leq\left\|\sum_{i=1}^{M}p_i\left[I-\frac{\eta_0\tau}{|\mathcal{D}_i|}\Theta_i^{t\tau+\bar{r}_i}+\Omega\left(e^{-\frac{\eta_0\tau}{|\mathcal{D}_i|}\Theta_i^{t\tau+\bar{r}_i}}\right)\right]\right\|_{\rm op}\left\|g\left(\theta_i^{t\tau}\right)\right\|_2\nonumber\\
	&\leq\left\|\sum_{i=1}^{M}p_i\left(I-\frac{\eta_0\tau}{|\mathcal{D}_i|}\Theta_i^{t\tau+\bar{r}_i}\right)\right\|_{\rm op}\left\|g\left(\theta_i^{t\tau}\right)\right\|_2+\sum_{i=1}^{M}p_i\left\|\Omega\left(e^{-\frac{\eta_0\tau}{|\mathcal{D}_i|}\Theta_i^{t\tau+\bar{r}_i}}\right)\right\|_{\rm op}\left\|g\left(\theta_i^{t\tau}\right)\right\|_2\nonumber\\
	&=\left\|I-\frac{\eta_0\tau}{|\mathcal{D}|}\sum_{i=1}^{M}\Theta_i^{t\tau+\bar{r}_i}\right\|_{\rm op}\left\|g\left(\theta_i^{t\tau}\right)\right\|_2+\sum_{i=1}^{M}p_i\left\|\Omega\left(e^{-\frac{\eta_0\tau}{|\mathcal{D}_i|}\Theta_i^{t\tau+\bar{r}_i}}\right)\right\|_{\rm op}\left\|g\left(\theta_i^{t\tau}\right)\right\|_2,
\end{align}
where the second step is obtained by employing the Taylor series expansion and $\Omega(e^{-A}) = \sum_{k=2}^{\infty}\frac{(-1)^k}{k!}A^k$. Then, we bound $\left\|I-\frac{\eta_0\tau}{|\mathcal{D}|}\sum_{i=1}^{M}\Theta_i^{t\tau+\bar{r}_i}\right\|_{\rm op}$ and $\sum_{i=1}^{M}p_i\left\|\Omega\left(e^{-\frac{\eta_0\tau}{|\mathcal{D}_i|}\Theta_i^{t\tau+\bar{r}_i}}\right)\right\|_{\rm op}$, respectively.

 $\big\|I-\frac{\eta_0\tau}{|\mathcal{D}|}\sum_{i=1}^{M}\Theta_i^{t\tau+\bar{r}_i}\big\|_{\rm op}$ can be rewritten as
\begin{align}\label{eq:equation68}
	\left\|I-\frac{\eta_0\tau}{|\mathcal{D}|}\sum_{i=1}^{M}\Theta_i^{t\tau+\bar{r}_i}\right\|_{\rm op}
	&=\left\|I-\frac{\eta_0\tau\Theta}{|\mathcal{D}|}+\frac{\eta_0\tau\Theta}{|\mathcal{D}|}-\frac{\eta_0\tau\Theta^0}{|\mathcal{D}|}+\frac{\eta_0\tau\Theta^0}{|\mathcal{D}|}-\frac{\eta_0\tau}{|\mathcal{D}|}\sum_{i=1}^{M}\Theta_i^{t\tau+\bar{r}_i}\right\|_{\rm op}\nonumber\\
	&\leq\left\|I-\frac{\eta_0\tau\Theta}{|\mathcal{D}|}\right\|_{\rm op}+\left\|\frac{\eta_0\tau\Theta}{|\mathcal{D}|}-\frac{\eta_0\tau\Theta^0}{|\mathcal{D}|}\right\|_{\rm op}+\left\|\frac{\eta_0\tau\Theta^0}{|\mathcal{D}|}-\frac{\eta_0\tau}{|\mathcal{D}|}\sum_{i=1}^{M}\Theta_i^{t\tau+\bar{r}_i}\right\|_{\rm op}\nonumber\\
	&= \left(1-\frac{\eta_0\tau\lambda_{m}}{|\mathcal{D}|}\right)+\frac{\eta_0\tau}{|\mathcal{D}|}\left\|\Theta-\Theta^0\right\|_{\rm op}+\frac{\eta_0\tau}{|\mathcal{D}|}\left\|\Theta^0-\sum_{i=1}^{M}\Theta_i^{t\tau+\bar{r}_i}\right\|_{\rm op}\nonumber\\
	&\leq \left(1-\frac{\eta_0\tau\lambda_{m}}{|\mathcal{D}|}\right)+\frac{\eta_0\tau}{|\mathcal{D}|}\left\|\Theta-\Theta^0\right\|_{F}+\frac{\eta_0\tau}{|\mathcal{D}|}\left\|\Theta^0-\sum_{i=1}^{M}\Theta_i^{t\tau+\bar{r}_i}\right\|_{\rm op}.
\end{align}

According to \citet[Section G.1]{lee2019wide}, there exists $n''$, such that the following event
\begin{equation}
	\left\|\Theta^0-\Theta\right\|_{F} \leq \frac{\lambda_m}{3} \label{eqn:n2}
\end{equation}
with probability at least $1-\frac{\delta_0}{5}$ hold for any $n\geq n''$. As for $\left\|\sum_{i=1}^{M}\Theta_i^{t\tau+\bar{r}_i}-\Theta^0\right\|_{\rm op}$, according to the definition of local NTK \eqref{equation21}, we can obtain
\begin{align}\label{eq:equation70}
	\left\|\sum_{i=1}^{M}\Theta_i^{t\tau+\bar{r}_i}-\Theta^0\right\|_{\rm op}
	&=\frac{1}{n}\left\|\sum_{i=1}^{M}J\left(\theta_i^{t\tau+\bar{r}_i}\right)^TJ_i\left(\theta_i^{t\tau+\bar{r}_i}\right)P_i-\sum_{i=1}^{M}J\left(\theta^{0}\right)^TJ_i\left(\theta_i^{0}\right)P_i\right\|_{\rm op}\nonumber\\
	&=\frac{1}{n}\left\|\sum_{i=1}^{M}\left[J\left(\theta_i^{t\tau+\bar{r}_i}\right)^TJ_i\left(\theta_i^{t\tau+\bar{r}_i}\right)P_i-J\left(\theta^{0}\right)^TJ_i\left(\theta_i^{0}\right)P_i\right]\right\|_{\rm op}\nonumber\\
	&\leq \frac{1}{n}\sum_{i=1}^{M}\left\|J\left(\theta_i^{t\tau+\bar{r}_i}\right)^TJ_i\left(\theta_i^{t\tau+\bar{r}_i}\right)-J\left(\theta^{0}\right)^TJ_i\left(\theta_i^{0}\right)\right\|_{\rm op}\left\|P_i\right\|_{\rm op}\nonumber\\
	&= \frac{1}{n}\sum_{i=1}^{M}\left\|J\left(\theta_i^{t\tau+\bar{r}_i}\right)^TJ_i\left(\theta_i^{t\tau+\bar{r}_i}\right)-J\left(\theta_i^{0}\right)^TJ_i\left(\theta_i^{0}\right)\right\|_{\rm op}\nonumber\\
	&= \frac{1}{n}\sum_{i=1}^{M}\left\|J\left(\theta_i^{t\tau+\bar{r}_i}\right)^T\left[J_i\left(\theta_i^{t\tau+\bar{r}_i}\right)-J_i\left(\theta_i^0\right)\right]+\left[J\left(\theta_i^{t\tau+\bar{r}_i}\right)^T-J\left(\theta_i^{0}\right)^T\right]J_i\left(\theta_i^{0}\right)\right\|_{\rm op}\nonumber\\
	&\leq \frac{1}{n}\sum_{i=1}^{M}\left[\left\|J\left(\theta_i^{t\tau+\bar{r}_i}\right)\right\|_{\rm op}\left\|J_i\left(\theta_i^{t\tau+\bar{r}_i}\right)-J_i\left(\theta_i^0\right)\right\|_{\rm op}+\left\|J\left(\theta_i^{t\tau+\bar{r}_i}\right)^T-J\left(\theta_i^{0}\right)^T\right\|_{\rm op}\left\|J_i\left(\theta_i^{0}\right)\right\|_{\rm op}\right]\nonumber\\
	&\leq \frac{1}{n}\sum_{i=1}^{M}\left[\left\|J\left(\theta_i^{t\tau+\bar{r}_i}\right)\right\|_{F}\left\|J_i\left(\theta_i^{t\tau+\bar{r}_i}\right)-J_i\left(\theta_i^0\right)\right\|_{F}+\left\|J\left(\theta_i^{t\tau+\bar{r}_i}\right)^T-J\left(\theta_i^{0}\right)^T\right\|_{F}\left\|J_i\left(\theta_i^{0}\right)\right\|_{F}\right]\nonumber\\
	&\leq 2C^2\sum_{i=1}^{M}\left\|\theta_i^{t\tau+\bar{r}_i}-\theta_i^{0}\right\|_2,
\end{align}
where the last step holds because of Lemma~\ref{lem:lemma1} and Lemma~\ref{lem:lemma2}. Plugging \eqref{eq:equation56} into \eqref{eq:equation70} yields
\begin{align}
	\left\|\sum_{i=1}^{M}\Theta_i^{t\tau+\bar{r}_i}-\Theta^0\right\|_{\rm op} 
	&\leq \frac{2\eta_0 \bar{r}_i C^3 q^{t}R_0}{\sqrt{n}|\mathcal{D}_i|}+ \frac{2\eta_0 \tau C^3 R_0\left(1-q^t\right)}{\sqrt{n}\left(1-q\right)} \nonumber\\
	&\leq \frac{2\eta_0 \tau C^3 R_0}{\sqrt{n}}+ \frac{2\eta_0 \tau C^3 R_0}{\sqrt{n}\left(1-q\right)} \nonumber\\
	&=\frac{4\eta_0 \tau C^3 R_0\left(2-q\right)}{\sqrt{n}\left(1-q\right)}\nonumber\\
	&\leq  \frac{\lambda_m}{3}, \label{eqn:n3}
\end{align}
where the second step holds because $\bar{r}_i\leq \tau$, $|\mathcal{D}_i|\geq1$ and $0\leq q<1$, and the last step holds when $n\geq n'''$ with $n'''=\frac{144\eta_0^2 \tau^2 C^6 R_0^2\left(2-q\right)^2}{\lambda_m^2\left(1-q\right)^2}$. Then, plugging \eqref{eqn:n2} and \eqref{eqn:n3} into \eqref{eq:equation68}, we can obtain
\begin{equation}\label{eq:equation60}
	\left\|I-\frac{\eta_0\tau}{|\mathcal{D}|}\sum_{i=1}^{M}\Theta_i^{t\tau+\bar{r}_i}\right\|_{\rm op} \leq 1-\frac{\eta_0\tau\lambda_m}{3|\mathcal{D}|}.
\end{equation}
To bound $\sum_{i=1}^{M}p_i\big\|\Omega\big(e^{-\frac{\eta_0\tau}{|\mathcal{D}_i|}\Theta_i^{t\tau+\bar{r}_i}}\big)\big\|_{\rm op}$, we first bound $\left\|\Theta_i^{t\tau+\bar{r}_i}\right\|_{\rm op}$. According to \eqref{equation21}, we have
\begin{align}\label{eq:equation54}
	\left\|\Theta_i^{t\tau+\bar{r}_i}\right\|_{\rm op}
	&= \frac{1}{n}\left\|J\left(\theta_i^{t\tau+\bar{r}_i}\right)^TJ_i\left(\theta_i^{t\tau+\bar{r}_i}\right)P_i\right\|_{\rm op}\nonumber\\
	&\leq \frac{1}{n}\left\|J\left(\theta_i^{t\tau+\bar{r}_i}\right)\right\|_F\left\|J_i\left(\theta_i^{t\tau+\bar{r}_i}\right)\right\|_F\left\|P_i\right\|_{\rm op}\nonumber\\
	&= \frac{1}{n}\left\|J\left(\theta_i^{t\tau+\bar{r}_i}\right)\right\|_F\left\|J_i\left(\theta_i^{t\tau+\bar{r}_i}\right)\right\|_F\nonumber\\
	&\leq C^2,
\end{align}
where the last step holds beacuse of Lemma~\ref{lem:lemma1} and Lemma~\ref{lem:lemma2}. Then, according to Lemma~\ref{lem:lemma3}, we can obtain 
\begin{align}\label{eq:equation45}
	\sum_{i=1}^{M}p_i\left\|\Omega\left(e^{-\frac{\eta_0\tau}{|\mathcal{D}_i|}\Theta_i^{t\tau+\bar{r}_i}}\right)\right\|_{\rm op} 
	&\leq \sum_{i=1}^{M}\frac{p_i}{2}\left(\frac{\eta_0\tau C^2}{|\mathcal{D}_i|}\right)^2e^{\frac{\eta_0\tau C^2}{|\mathcal{D}_i|}}\nonumber\\
	&=\sum_{i=1}^{M}\frac{\eta_0^2\tau^2C^4}{2|\mathcal{D}_i||\mathcal{D}|}e^{\frac{\eta_0\tau C^2}{|\mathcal{D}_i|}}\nonumber\\
	&\leq \sum_{i=1}^{M}\frac{\eta_0^2\tau^2C^4}{2|\mathcal{D}|}e^{\eta_0\tau C^2}\nonumber\\
	&=\frac{M\eta_0^2\tau^2C^4}{2|\mathcal{D}|}e^{\eta_0\tau C^2}\nonumber\\
	&\leq \frac{\eta_0^2\tau^2C^4}{2}e^{\eta_0\tau C^2}.
\end{align}
Further plugging \eqref{eq:equation60} and \eqref{eq:equation45} to \eqref{eq:equation42}, we have
\begin{align}\label{eq:equation46}
	\left\|\sum_{i=1}^{M}p_ig\left(\theta_i^{t\tau+\tau}\right)\right\|_2\leq\left(1-\frac{\eta_0\tau\lambda_m}{3|\mathcal{D}|}+\frac{\eta_0^2\tau^2C^4}{2}e^{\eta_0\tau C^2}\right)\left\|g\left(\theta_i^{t\tau}\right)\right\|_2 \triangleq q\left\|g\left(\theta_i^{t\tau}\right)\right\|_2
\end{align}

\paragraph{2) Bounding $\big\|\sum_{i=1}^{M}p_i\Omega_i\big\|_2$:}  By defining $
	\Gamma\left(\beta\right) = g\big(\theta^{\left(t+1\right)\tau}+\beta\Delta \theta_i^{\left(t+1\right)\tau}\big), \beta \in \left[0,1\right]
$, 
we can obtain
\begin{equation}
	\Gamma\left(0\right) = g\left(\theta^{\left(t+1\right)\tau}\right), \quad  \Gamma'\left(0\right) = \nabla g\left(\theta^{\left(t+1\right)\tau}\right)\Delta \theta_i^{\left(t+1\right)\tau}, \quad 	\Gamma\left(1\right) = g\left(\theta_i^{t\tau+\tau}\right).
\end{equation}
We further define
\begin{align}\label{eq:equation200}
	u\left(\beta\right) = \Gamma\left(\beta\right) + \left(1-\beta\right)\Gamma'\left(\beta\right).
\end{align}
Then, we can obtain
\begin{align}\label{eq:equation777}
	\left\|u\left(1\right) - u\left(0\right)\right\|_2 
	&= \left\|\Gamma\left(1\right) - \Gamma\left(0\right) -\Gamma'\left(0\right)\right\|_2\nonumber\\
	&= \left\|g\left(\theta_i^{t\tau+\tau}\right) - g\left(\theta^{\left(t+1\right)\tau}\right) -\nabla g\left(\theta^{\left(t+1\right)\tau}\right)\Delta \theta_i^{\left(t+1\right)\tau}\right\|_2\nonumber\\
	&= \left\|\Omega_i\right\|_2.
\end{align}
According to mean value inequality, we have
\begin{align}\label{eq:equation77}
	\left\|u\left(1\right) - u\left(0\right)\right\|_2 \leq \sup_{\substack{0\leq\beta\leq1}}\left\|u'\left(\beta\right)\right\|_{2}\left(1-0\right).
\end{align}
Combining \eqref{eq:equation777} and \eqref{eq:equation77} yields
\begin{align}\label{eq:equation301}
	\left\|\Omega_i\right\|_2\leq \sup_{\substack{0\leq\beta\leq1}}\left\|u'\left(\beta\right)\right\|_2.
\end{align}
Taking the derivative on both side of \eqref{eq:equation200} yields
\begin{align} 
	u'\left(\beta\right) 
	&= \Gamma'\left(\beta\right) - \Gamma'\left(\beta\right) + \left(1-\beta\right)\Gamma''\left(\beta\right)\nonumber\\
	&= \left(1-\beta\right)\Gamma''\left(\beta\right) \nonumber\\
	&= \left(1-\beta\right)\left(\Delta \theta_i^{\left(t+1\right)\tau}\right)^T\nabla^2 g\left(\theta^{\left(t+1\right)\tau}+\beta\Delta  \theta_i^{\left(t+1\right)\tau}\right) \Delta \theta_i^{\left(t+1\right)\tau}. \label{eqn:u'}
\end{align}
By plugging~\eqref{eqn:u'} into \eqref{eq:equation301}, and then \eqref{eq:equation301} into $\big\|\sum_{i=1}^{M}p_i\Omega_i\big\|_2$, we can obtain
\begin{align}\label{eq:equation300}
	\left\|\sum_{i=1}^{M}p_i\Omega_i\right\|_2 
	&\leq \sum_{i=1}^{M}p_i\left\|\Omega_i\right\|_2 \nonumber\\
	&\leq \sum_{i=1}^{M}p_i\sup_{\substack{0\leq\beta\leq1}}\left\|u'\left(\beta\right)\right\|_2\nonumber\\
	&=\sum_{i=1}^{M}p_i\left(1-\beta\right)\sup_{\substack{0\leq\beta\leq1}}\left\|\left(\Delta \theta_i^{\left(t+1\right)\tau}\right)^T\nabla^2 g\left(\theta^{\left(t+1\right)\tau}+\beta\Delta  \theta_i^{\left(t+1\right)\tau}\right) \Delta \theta_i^{\left(t+1\right)\tau}\right\|_2\nonumber\\
	&\leq \sum_{i=1}^{M}p_i\left(1-\beta\right)\left\|\Delta \theta_i^{\left(t+1\right)\tau}\right\|_2\sup_{\substack{0\leq\beta\leq1}}\left\|\nabla^2 g\left(\theta^{\left(t+1\right)\tau}+\beta\Delta  \theta_i^{\left(t+1\right)\tau}\right)\right\|_{\rm op} \left\|\Delta \theta_i^{\left(t+1\right)\tau}\right\|_2\nonumber\\
	&\overset{\sf(a)}{\leq} \sum_{i=1}^{M}p_i\left(1-\beta\right)\left\|\Delta \theta_i^{\left(t+1\right)\tau}\right\|_2 C_1\sqrt{n} \left\|\Delta \theta_i^{\left(t+1\right)\tau}\right\|_2\nonumber\\
	&\overset{\sf(b)}{\leq} \sum_{i=1}^{M}p_i\left(1-\beta\right) \frac{2\eta_0\tau CR_0}{\sqrt{n}\left(1-q\right)} \cdot C_1\sqrt{n} \cdot\frac{2\eta_0\tau CR_0}{\sqrt{n}\left(1-q\right)}\nonumber\\
	&\overset{\sf(c)}{\leq} \frac{4\eta_0^2\tau^2 C^2R_0^2C_1}{\sqrt{n}\left(1-q\right)^2}\nonumber\\
	&= \frac{2\eta_0\tau CC_1R_0\zeta}{\left(1-q\right)}.
\end{align}
where step $\sf (a)$ applies $\frac{1}{\sqrt{n}}\left\|\nabla^2g\left(\theta\right)\right\|_{\rm op}\leq C_1$, which is proved by \citet[Lemma  1]{jacot2020asymptotic} under Assumption~\ref{ass:activation}, step $\sf (b)$ comes from \eqref{eq:equation39}, step $\sf (c)$ holds because $\beta\in[0,1]$, and the last step comes from \eqref{eq:equation165}. Plugging \eqref{eq:equation46} and \eqref{eq:equation300} into \eqref{eq:equation41}:
\begin{align}\label{eq:equation61}
	\left\|g\left(\theta^{\left(t+1\right)\tau}\right)\right\|_2
	&\leq q\left\|g\left(\theta^{t\tau}\right)\right\|_2+\frac{2\eta_0\tau CC_1R_0\zeta}{\left(1-q\right)}\nonumber\\
	&\overset{\sf (a)}{=} q\left[q^{t}\left(R_0-\frac{2\eta_0\tau CC_1R_0\zeta}{\left(1-q\right)^2}\right)+\frac{2\eta_0\tau CC_1R_0\zeta}{\left(1-q\right)^2}\right] + \left(1-q\right)\frac{2\eta_0\tau CC_1R_0\zeta}{\left(1-q\right)^2}\nonumber\\
	&=q^{t+1}\left(R_0-\frac{2\eta_0\tau CC_1R_0\zeta}{\left(1-q\right)^2}\right)+\frac{2\eta_0\tau CC_1R_0\zeta}{\left(1-q\right)^2},
\end{align}
where step $\sf (a)$ applies induction hypothesis \eqref{equation39}. At this point, we have proven \eqref{eq:equation64}. Recall that \eqref{eq:equation65} have also been proven and we require $n \geq n', n'', n'''$ during the derivations, and hence the induction hypotheses \eqref{equation39} and \eqref{equation40} hold for $n \geq N\triangleq\max \{n', n'', n'''\}$.


Noting that we also require $0\leq q<1$ to complete the proof, we prove there exists $\eta_0>0$ such that $0\leq q<1$ in the following. Referring to \eqref{eq:equation46}, we rewrite $q$ as a function of $\eta_0$ as
\begin{equation}
	q\left(\eta_0\right) = 1-\frac{\eta_0\tau\lambda_m}{3|\mathcal{D}|}+\frac{\eta_0^2\tau^2C^4}{2}e^{\eta_0\tau C^2}. \label{eqn:q}
\end{equation}
Taking the derivative of $q\left(\eta_0\right)$ with respect to $\eta_0$, we obtain
\begin{equation}
	q'\left(\eta_0\right)=-\frac{\tau\lambda_m}{3|\mathcal{D}|}+\eta_0\tau^2C^4e^{\eta_0\tau C^2}+\frac{\eta_0^2\tau^3C^6}{2}e^{\eta_0\tau C^2}.
\end{equation}
Further taking the second derivative, we have
\begin{equation}
	q''\left(\eta_0\right)=\tau^2C^4e^{\eta_0\tau C^2}+\eta_0\tau^3C^6e^{\eta_0\tau C^2}+\frac{\eta_0^2\tau^4C^8}{2}e^{\eta_0\tau C^2} > 0.
\end{equation}
Therefore, $q'\left(\eta_0\right)$ is monotonically increasing. According to Assumption \ref{ass:ntkfullrank}, we have $\lambda_m > 0$ and hence it obvious that $q'\left(0\right)<0$ and $\lim\limits_{\eta_0\to\infty} q'\left(\eta_0\right)> 0$, and hence there exists $\eta_0'>0$ such that $q'\left(\eta_0'\right)=0$ holds. Consequently, $q\left(\eta_0\right)$ is monotonically decreasing on $\left(0,\eta_0'\right]$, and monotonically increasing on $\left(\eta_0', \infty\right)$. Additionally, from \eqref{eqn:q}, we have $\lim\limits_{\eta_0\to 0} q\left(\eta_0\right) = 1$. Consequently, if $q\left(\eta_0'\right)\geq0$, $0\leq q\left(\eta_0\right)<1$ holds for $0<\eta_0\leq\eta_0'$. Otherwise, if $q\left(\eta_0'\right)<0$, there exists $\eta_0''\in \left(0, \eta_0'\right)$ such that $q\left(\eta_0''\right)=0$ holds. Then, $0\leq q\left(\eta_0\right)<1$ holds for $0<\eta_0\leq\eta_0''$. To sum up, as long as $0<\eta_0\leq\rm{min}\{\eta_0',\eta_0''\}$, $0\leq q\left(\eta_0\right)<1$ holds. 
\subsection{Bounding the Variation on Global and Local NTKs }
We continue to prove \eqref{eq:ntk} and \eqref{eq:local ntk} in Theorem~\ref{theorem1}. Referring to \eqref{eqn:ntk}, we have
\begin{align}
	\left\|\Theta^{t\tau}-\Theta^0\right\|_{F}
	&=\frac{1}{n}\left\|J\left(\theta^{t\tau}\right)^TJ\left(\theta^{t\tau}\right)-J\left(\theta^{0}\right)^TJ\left(\theta^{0}\right)\right\|_{F}\nonumber\\
	&=\frac{1}{n}\left\|\left[J\left(\theta^{t\tau}\right)^T-J\left(\theta^{0}\right)^T\right]J\left(\theta^{t\tau}\right)-J\left(\theta^{0}\right)^T\left[J\left(\theta^{t\tau}\right)-J\left(\theta^{0}\right)\right]\right\|_{F}\nonumber\\
	&\leq \frac{1}{n}\left\|J\left(\theta^{t\tau}\right)-J\left(\theta^{0}\right)\right\|_F\left\|J\left(\theta^{t\tau}\right)\right\|_F+\frac{1}{n}\left\|J\left(\theta^{0}\right)\right\|_F\left\|J\left(\theta^{t\tau}\right)-J\left(\theta^{0}\right)\right\|_{F}\nonumber\\
	&\leq 2C^2\left\|\theta^{t\tau}-\theta^{0}\right\|_2\nonumber\\
	&\leq \frac{2\eta_0 \tau C^3 R_0\left(1-q^t\right)}{\sqrt{n}\left(1-q\right)}.
\end{align}
where the fourth step holds because of Lemma~\ref{lem:lemma1} and the last step holds because of \eqref{equation40}. 

Referring to \eqref{equation21}, we have
\begin{align}
	\left\|\Theta_i^{t\tau+r} - \Theta_i^0\right\|_F &= \frac{1}{n}\left\|J\left(\theta_i^{t\tau+r}\right)^TJ_i\left(\theta_i^{t\tau+r}\right)P_i - J\left(\theta_i^{0}\right)^TJ_i\left(\theta_i^{0}\right)P_i\right\|_F\nonumber\\
	&\leq \frac{1}{n}\left\|J\left(\theta_i^{t\tau+r}\right)^TJ_i\left(\theta_i^{t\tau+r}\right) - J\left(\theta_i^{0}\right)^TJ_i\left(\theta_i^{0}\right)\right\|_F\left\|P_i\right\|_F\nonumber\\
	&\leq \frac{\sqrt{k|\mathcal{D}_i|}}{n}\left\|J\left(\theta_i^{t\tau+r}\right)^TJ_i\left(\theta_i^{t\tau+r}\right) - J\left(\theta_i^{0}\right)^TJ_i\left(\theta_i^{0}\right)\right\|_F\nonumber\\
	&= \frac{\sqrt{k|\mathcal{D}_i|}}{n}\left\|\left[J\left(\theta_i^{t\tau+r}\right)^T-J\left(\theta_i^{0}\right)^T\right]J_i\left(\theta_i^{t\tau+r}\right) + J\left(\theta_i^{0}\right)^T\left[J_i\left(\theta_i^{t\tau+r}\right)-J_i\left(\theta_i^{0}\right)\right]\right\|_F\nonumber\\
	&\leq \frac{\sqrt{k|\mathcal{D}_i|}}{n}\left\|J\left(\theta_i^{t\tau+r}\right)-J\left(\theta_i^{0}\right)\right\|_F\left\|J_i\left(\theta_i^{t\tau+r}\right)\right\|_F + \frac{\sqrt{k|\mathcal{D}_i|}}{n}\left\|J\left(\theta_i^{0}\right)\right\|_F\left\|J_i\left(\theta_i^{t\tau+r}\right)-J_i\left(\theta_i^{0}\right)\right\|_F\nonumber\\
	&\overset{\sf (a)}{\leq} 2C^2\sqrt{k|\mathcal{D}_i|}\left\|\theta_i^{t\tau+r} - \theta_i^{0}\right\|_2\nonumber\\
	&\leq 
	\frac{2\eta_0 r q^{t}C^3R_0\sqrt{k}}{\sqrt{n|\mathcal{D}_i|}}+ \frac{2\eta_0 \tau C^3 R_0\left(1-q^t\right)\sqrt{k|\mathcal{D}_i|}}{\left(1-q\right)\sqrt{n}},
\end{align}
where step $\sf (a)$ employs Lemma~\ref{lem:lemma1} and Lemma~\ref{lem:lemma2}, and the last step holds because of \eqref{eq:equation56}. We have now completed the proof of Theorem~\ref{theorem1}.

\section{Proof of Theorem~\ref{theorem2}}\label{appendixB}
To prove Theorem~\ref{theorem2}, we first bound $\left\|g^{\rm lin}\left(\theta^{\left(t+1\right)\tau}\right)-g\left(\theta^{\left(t+1\right)\tau}\right)\right\|_2$. According to \eqref{eq:equation25}, we have
\begin{equation}
	g^{\rm lin}\left(\theta_i^{t\tau+r}\right)=g\left(\theta^0\right)+J\left(\theta^0\right)^T\left(\theta_i^{t\tau+r}-\theta^0\right).
\end{equation}
Taking the derivative with respect to $\theta_i^{t\tau+r}$ on both sides yields
\begin{equation}
	J^{\rm lin}\left(\theta_i^{t\tau+r}\right) = J\left(\theta^0\right).
\end{equation}
Then, we consider
\begin{align}\label{eq:equation63}
	&\frac{d}{dr}\left(e^{\frac{\eta_0r}{|\mathcal{D}_i|}\Theta_i^0}\left[g^{\rm lin}\left(\theta_i^{t\tau+r}\right)-g\left(\theta_i^{t\tau+r}\right)\right]\right)\nonumber\\
	=~&\frac{\eta_0\Theta_i^0}{|\mathcal{D}_i|}e^{\frac{\eta_0r}{|\mathcal{D}_i|}\Theta_i^0}\left[g^{\rm lin}\left(\theta_i^{t\tau+r}\right)-g\left(\theta_i^{t\tau+r}\right)\right] \nonumber\\
	&+ \frac{\eta_0}{|\mathcal{D}_i|} e^{\frac{\eta_0r}{|\mathcal{D}_i|}\Theta_i^0}\left[-\frac{1}{n}J(\theta^0)^TJ_i\left(\theta_i^0\right)g_i^{\rm lin}\left(\theta_i^{t\tau+r}\right)+\frac{1}{n}J\left(\theta_i^{t\tau+r}\right)^TJ_i\left(\theta_i^{t\tau+r}\right)g_i\left(\theta_i^{t\tau+r}\right)\right]\nonumber\\
	=~&\frac{\eta_0\Theta_i^0}{|\mathcal{D}_i|}e^{\frac{\eta_0r}{|\mathcal{D}_i|}\Theta_i^0}\left(g^{\rm lin}\left(\theta_i^{t\tau+r}\right)-g\left(\theta_i^{t\tau+r}\right)\right) \nonumber\\
	&+ \frac{\eta_0}{|\mathcal{D}_i|} e^{\frac{\eta_0r}{|\mathcal{D}_i|}\Theta_i^0}\left[-\frac{1}{n}J(\theta^0)^TJ_i\left(\theta_i^0\right)P_ig^{\rm lin}\left(\theta_i^{t\tau+r}\right)+\frac{1}{n}J\left(\theta_i^{t\tau+r}\right)^TJ_i\left(\theta_i^{t\tau+r}\right)P_ig\left(\theta_i^{t\tau+r}\right)\right]\nonumber\\
	=~&\frac{\eta_0\Theta_i^0}{|\mathcal{D}_i|}e^{\frac{\eta_0r}{|\mathcal{D}_i|}\Theta_i^0}\left[g^{\rm lin}\left(\theta_i^{t\tau+r}\right)-g\left(\theta_i^{t\tau+r}\right)\right] +\frac{\eta_0}{|\mathcal{D}_i|} e^{\frac{\eta_0r}{|\mathcal{D}_i|}\Theta_i^0}\left[-\Theta_i^0g^{\rm lin}\left(\theta_i^{t\tau+r}\right)+\Theta_i^{t\tau+r}g\left(\theta_i^{t\tau+r}\right)\right]\nonumber\\
	=~&\frac{\eta_0}{|\mathcal{D}_i|}e^{\frac{\eta_0r}{|\mathcal{D}_i|}\Theta_i^0}\left(\Theta_i^{t\tau+r}-\Theta_i^0\right)g\left(\theta_i^{t\tau+r}\right).
\end{align}
By integrating both sides from $0$ to $\tau$, we obtain
\begin{equation}\label{eq:equation55}
	g^{\rm lin}\left(\theta_i^{t\tau+\tau}\right)-g\left(\theta_i^{t\tau+\tau}\right)=e^{-\frac{\eta_0\tau}{|\mathcal{D}_i|}\Theta_i^0}\left[g^{\rm lin}\left(\theta^{t\tau}\right)-g\left(\theta^{t\tau}\right)\right]+\frac{\eta_0}{|\mathcal{D}_i|}e^{-\frac{\eta_0\tau}{|\mathcal{D}_i|}\Theta_i^0}\int_{0}^{\tau}e^{\frac{\eta_0r}{|\mathcal{D}_i|}\Theta_i^0}\left(\Theta_i^{t\tau+r}-\Theta_i^0\right)g\left(\theta_i^{t\tau+r}\right)dr.
\end{equation}
The aggregation of linear local models yields
\begin{equation}
	g^{\rm lin}\left(\theta^{\left(t+1\right)\tau}\right)=\sum_{i=1}^{M}p_ig^{\rm lin}\left(\theta_i^{t\tau+\tau}\right).
\end{equation}
Further considering equation~\eqref{eq:equation20}, we have
\begin{equation}
	g^{\rm lin}\left(\theta^{\left(t+1\right)\tau}\right)-g\left(\theta^{\left(t+1\right)\tau}\right)=\sum_{i=1}^{M}p_i\left[g^{\rm lin}\left(\theta_i^{t\tau+\tau}\right)-g\left(\theta_i^{t\tau+\tau}\right)\right]+\sum_{i=1}^{M}p_i\Omega_i.
\end{equation}
By taking the norm on both sides, we obtain
\begin{align}\label{eq:equation58}
	&\left\|g^{\rm lin}\left(\theta^{\left(t+1\right)\tau}\right)-g\left(\theta^{\left(t+1\right)\tau}\right)\right\|_2\nonumber\\
	&=\left\|\sum_{i=1}^{M}p_i\left[g^{\rm lin}\left(\theta_i^{t\tau+\tau}\right)-g\left(\theta_i^{t\tau+\tau}\right)\right]+\sum_{i=1}^{M}p_i\Omega_i\right\|_2\nonumber\\
	&\leq\left\|\sum_{i=1}^{M}p_i\left[g^{\rm lin}\left(\theta_i^{t\tau+\tau}\right)-g\left(\theta_i^{t\tau+\tau}\right)\right]\right\|_2+\left\|\sum_{i=1}^{M}p_i\Omega_i\right\|_2\nonumber\\
	&\leq\left\|\sum_{i=1}^{M}p_ie^{-\frac{\eta_0\tau}{|\mathcal{D}_i|}\Theta_i^0}\left[g^{\rm lin}\left(\theta^{t\tau}\right)-g\left(\theta^{t\tau}\right)\right]\right\|_{2} + \left\|\sum_{i=1}^{M}p_i\frac{\eta_0}{|\mathcal{D}_i|}e^{-\frac{\eta_0\tau}{|\mathcal{D}_i|}\Theta_i^0}\int_{0}^{\tau}e^{\frac{\eta_0r}{|\mathcal{D}_i|}\Theta_i^0}\left(\Theta_i^{t\tau+r}-\Theta_i^0\right)g\left(\theta_i^{t\tau+r}\right)dr\right\|_{\rm op} \nonumber \\
	&\quad + \left\|\sum_{i=1}^{M}p_i\Omega_i\right\|_2,
\end{align}
where the last step is obtained by substituting~\eqref{eq:equation55}, and $\big\|\sum_{i=1}^{M}p_i\Omega_i\big\|_2$ is bounded in \eqref{eq:equation300}. In the following, we bound  $ \big\|\sum_{i=1}^{M}p_ie^{-\frac{\eta_0\tau}{|\mathcal{D}_i|}\Theta_i^0}\left(g^{\rm lin}\left(\theta^{t\tau}\right)-g\left(\theta^{t\tau}\right)\right)\big\|_{2}$ and $\big\|\sum_{i=1}^{M}p_i\frac{\eta_0}{|\mathcal{D}_i|}e^{-\frac{\eta_0\tau}{|\mathcal{D}_i|}\Theta_i^0}\int_{0}^{\tau}e^{\frac{\eta_0r}{|\mathcal{D}_i|}\Theta_i^0}\left(\Theta_i^{t\tau+r}-\Theta_i^0\right)g\left(\theta_i^{t\tau+r}\right)dr\big\|_{\rm op}$, respectively, and finally bound $\left\|f^{\rm lin}\left(\theta^{t\tau+r}\right)-f\left(\theta^{t\tau+r}\right)\right\|_2$ as well as $\left\|f_i^{\rm lin}\left(\theta_i^{t\tau+r}\right)-f_i\left(\theta_i^{t\tau+r}\right)\right\|_2$.

\subsection{Bounding $\big\|\sum_{i=1}^{M}p_ie^{-\frac{\eta_0\tau}{|\mathcal{D}_i|}\Theta_i^0}\left(g^{\rm lin}\left(\theta^{t\tau}\right)-g\left(\theta^{t\tau}\right)\right)\big\|_{2}$}
By taking the Taylor series expansion of $e^{-\frac{\eta_0\tau}{|\mathcal{D}_i|}\Theta_i^0}$, we obtain
\begin{align}\label{eq:equation59}
	&\left\|\sum_{i=1}^{M}p_ie^{-\frac{\eta_0\tau}{|\mathcal{D}_i|}\Theta_i^0}\left(g^{\rm lin}\left(\theta^{t\tau}\right)-g\left(\theta^{t\tau}\right)\right)\right\|_{2}\nonumber\\
	&\leq \left\|\sum_{i=1}^{M}p_ie^{-\frac{\eta_0\tau}{|\mathcal{D}_i|}\Theta_i^0}\right\|_{\rm op}\left\|g^{\rm lin}\left(\theta^{t\tau}\right)-g\left(\theta^{t\tau}\right)\right\|_{2}\nonumber\\
	&= \left\|\sum_{i=1}^{M}p_i\left[I-\frac{\eta_0\tau}{|\mathcal{D}_i|}\Theta_i^0+\Omega\left(e^{-\frac{\eta_0\tau}{|\mathcal{D}_i|}\Theta_i^0}\right)\right]\right\|_{\rm op}\left\|g^{\rm lin}\left(\theta^{t\tau}\right)-g\left(\theta^{t\tau}\right)\right\|_{2}\nonumber\\
	&= \left\|I-\frac{\eta_0\Theta^0\tau}{|\mathcal{D}|}\right\|_{\rm op}\left\|g^{\rm lin}\left(\theta^{t\tau}\right)-g\left(\theta^{t\tau}\right)\right\|_{2}+\sum_{i=1}^{M}p_i\left\|\Omega\left(e^{-\frac{\eta_0\tau}{|\mathcal{D}_i|}\Theta_i^0}\right)\right\|_{\rm op}\left\|g^{\rm lin}\left(\theta^{t\tau}\right)-g\left(\theta^{t\tau}\right)\right\|_{2}.
\end{align}
Next, we bound $\left\|I-\frac{\eta_0\Theta^0\tau}{|\mathcal{D}|}\right\|_{\rm op}$ and $\sum_{i=1}^{M}p_i\left\|\Omega\left(e^{-\frac{\eta_0\tau}{|\mathcal{D}_i|}\Theta_i^0}\right)\right\|_{\rm op}$, respectively.
\begin{align}\label{eq:equation84}
	\left\|I-\frac{\eta_0\Theta^0\tau}{|\mathcal{D}|}\right\|_{\rm op}
	&=\left\|I-\frac{\eta_0\Theta\tau}{|\mathcal{D}|}+\frac{\eta_0\Theta\tau}{|\mathcal{D}|}-\frac{\eta_0\Theta^0\tau}{|\mathcal{D}|}\right\|_{\rm op}\nonumber\\
	&\leq \left\|I-\frac{\eta_0\Theta\tau}{|\mathcal{D}|}\right\|_{\rm op}+\left\|\frac{\eta_0\Theta\tau}{|\mathcal{D}|}-\frac{\eta_0\Theta^0\tau}{|\mathcal{D}|}\right\|_{\rm op}\nonumber\\
	&= \left(1-\frac{\eta_0\tau\lambda_m}{|\mathcal{D}|}\right)+\left\|\frac{\eta_0\Theta\tau}{|\mathcal{D}|}-\frac{\eta_0\Theta^0\tau}{|\mathcal{D}|}\right\|_{\rm op}\nonumber\\
	&\leq \left(1-\frac{\eta_0\tau\lambda_m}{|\mathcal{D}|}\right)+\left\|\frac{\eta_0\Theta\tau}{|\mathcal{D}|}-\frac{\eta_0\Theta^0\tau}{|\mathcal{D}|}\right\|_{F}\nonumber\\
	&\overset{\sf (a)}{\leq} \left(1-\frac{\eta_0\tau\lambda_m}{|\mathcal{D}|}\right)+\frac{\eta_0\tau\lambda_m}{3|\mathcal{D}|}\nonumber\\
	&\leq 1-\frac{\eta_0\tau\lambda_m}{3|\mathcal{D}|}
\end{align}
where step $\sf (a)$ applies \eqref{eqn:n2}. To bound $\sum_{i=1}^{M}p_i\big\|\Omega\big(e^{-\frac{\eta_0\tau}{|\mathcal{D}_i|}\Theta_i^0}\big)\big\|_{\rm op}$, we first bound $\left\|\Theta_i^0\right\|_{\rm op}$  as
\begin{align}
	\left\|\Theta_i^0\right\|_{\rm op}
	&=\left\|\frac{1}{n}J(\theta^0)^TJ_i(\theta_i^0)P_i\right\|_{\rm op}\nonumber\\
	&\leq\frac{1}{n}\left\|J(\theta^0)\right\|_{F}\left\|J_i(\theta_i^0)\right\|_{F}\left\|P_i\right\|_{\rm op}\nonumber\\
	&=\frac{1}{n}\left\|J(\theta^0)\right\|_{F}\left\|J_i(\theta_i^0)\right\|_{F}\nonumber\\
	&\leq C^2,
\end{align}
where the last step applies Lemmas~\ref{lem:lemma1} and \ref{lem:lemma2}. Then, we have
\begin{align}\label{eq:equation86}
	\sum_{i=1}^{M}p_i\left\|\Omega\left(e^{-\frac{\eta_0\tau}{|\mathcal{D}_i|}\Theta_i^0}\right)\right\|_{\rm op}
	&\leq \sum_{i=1}^{M}\frac{p_i}{2}\left(\frac{\eta_0\tau C^2}{|\mathcal{D}_i|}\right)^2e^{\frac{\eta_0\tau C^2}{|\mathcal{D}_i|}}\nonumber\\
	&=\sum_{i=1}^{M}\frac{\eta_0^2\tau^2C^4}{2|\mathcal{D}_i||\mathcal{D}|}e^{\frac{\eta_0\tau C^2}{|\mathcal{D}_i|}}\nonumber\\
	&\leq \sum_{i=1}^{M}\frac{\eta_0^2\tau^2C^4}{2|\mathcal{D}|}e^{\eta_0\tau C^2}\nonumber\\
	&=\frac{M\eta_0^2\tau^2C^4}{2|\mathcal{D}|}e^{\eta_0\tau C^2}\nonumber\\
	&\leq \frac{\eta_0^2\tau^2C^4}{2}e^{\eta_0\tau C^2}.
\end{align}
where the first step employs Lemma~\ref{lem:lemma3}. Further, plugging \eqref{eq:equation84} and \eqref{eq:equation86} into \eqref{eq:equation59} yields
\begin{equation}\label{eq:equation87}
	\left\|\sum_{i=1}^{M}p_ie^{-\frac{\eta_0\tau}{|\mathcal{D}_i|}\Theta_i^0}\left(g^{\rm lin}\left(\theta^{t\tau}\right)-g\left(\theta^{t\tau}\right)\right)\right\|_{2} \leq q\left\|g^{\rm lin}\left(\theta^{t\tau}\right)-g\left(\theta^{t\tau}\right)\right\|_{2}.
\end{equation}

\subsection{Bounding $\big\|\sum_{i=1}^{M}p_i\frac{\eta_0}{|\mathcal{D}_i|}e^{-\frac{\eta_0\tau}{|\mathcal{D}_i|}\Theta_i^0}\int_{0}^{\tau}e^{\frac{\eta_0r}{|\mathcal{D}_i|}\Theta_i^0}\left(\Theta_i^{t\tau+r}-\Theta_i^0\right)g\left(\theta_i^{t\tau+r}\right)dr\big\|_{\rm op}$}
Considering $p_i=\frac{|\mathcal D_i|}{|\mathcal D|}$, we can obtain
\begin{align}\label{eq:equation76}
	&\left\|\sum_{i=1}^{M}p_i\frac{\eta_0}{|\mathcal{D}_i|}e^{-\frac{\eta_0\tau}{|\mathcal{D}_i|}\Theta_i^0}\int_{0}^{\tau}e^{\frac{\eta_0r}{|\mathcal{D}_i|}\Theta_i^0}\left(\Theta_i^{t\tau+r}-\Theta_i^0\right)g\left(\theta_i^{t\tau+r}\right)dr\right\|_{\rm op}\nonumber\\
	&=\left\|\sum_{i=1}^{M}\frac{\eta_0}{|\mathcal{D}|}e^{-\frac{\eta_0\tau}{|\mathcal{D}_i|}\Theta_i^0}\int_{0}^{\tau}e^{\frac{\eta_0r}{|\mathcal{D}_i|}\Theta_i^0}\left(\Theta_i^{t\tau+r}-\Theta_i^0\right)g\left(\theta_i^{t\tau+r}\right)dr\right\|_{\rm op}\nonumber\\
	&\overset{\sf (a)}{=}\left\|\sum_{i=1}^{M}\frac{\eta_0\tau}{|\mathcal{D}|}\left(\Theta_i^{t\tau+\tilde{r}_i}-\Theta_i^0\right)e^{\frac{\eta_0\left(\tilde{r}_i-\tau\right)}{|\mathcal{D}_i|}\Theta_i^0}g\left(\theta_i^{t\tau+\tilde{r}_i}\right)\right\|_{2}\nonumber\\
	&\leq\sum_{i=1}^{M}\frac{\eta_0\tau}{|\mathcal{D}|}\left\|\left(\Theta_i^{t\tau+\tilde{r}_i}-\Theta_i^0\right)\right\|_{\rm op}\left\|e^{\frac{\eta_0\left(\tilde{r}_i-\tau\right)}{|\mathcal{D}_i|}\Theta_i^0}\right\|_{\rm op}\left\|g\left(\theta_i^{t\tau+\tilde{r}_i}\right)\right\|_{2}\nonumber\\
	&\overset{\sf (b)}{\leq} \sum_{i=1}^{M}\frac{\eta_0\tau}{|\mathcal{D}|}\left\|\left(\Theta_i^{t\tau+\tilde{r}_i}-\Theta_i^0\right)\right\|_{\rm op}\left\|g\left(\theta_i^{t\tau+\tilde{r}_i}\right)\right\|_{2}\nonumber\\
	&\overset{\sf (c)}{\leq} \sum_{i=1}^{M}\frac{\eta_0\tau}{|\mathcal{D}|}\left\|\left(\Theta_i^{t\tau+\tilde{r}_i}-\Theta_i^0\right)\right\|_{\rm op}\left\|e^{-\frac{\eta_0\tau}{|\mathcal{D}_i|}\Theta_i^{t\tau+r'_i}}\right\|_{\rm op}\left\|g\left(\theta_i^{t\tau}\right)\right\|_{2}\nonumber\\
	&\leq \sum_{i=1}^{M}\frac{\eta_0\tau}{|\mathcal{D}|}\left\|\left(\Theta_i^{t\tau+\tilde{r}_i}-\Theta_i^0\right)\right\|_{\rm op}\left\|g\left(\theta_i^{t\tau}\right)\right\|_{2},
\end{align}
where step $\sf (a)$ holds according to the mean value theorem of integrals and $\tilde{r}_i \in \left(0,\tau\right)$, step $\sf (b)$ holds because $\Theta_i^{0}$ is not full rank according to the definition of local NTK \eqref{equation21} and hence $\big\|e^{-\frac{\eta_0\left(\tilde{r}_i-\tau\right)}{|\mathcal{D}_i|}\Theta_i^{0}}\big\|_{\rm op}\leq e^{-\frac{\eta_0\left(\tilde{r}_i-\tau\right)}{|\mathcal{D}_i|}\lambda_{\min}\left(\Theta_i^{0}\right)}= 1$, step $\sf(c)$ comes from \eqref{eq:equation49} and $r'_i\in \left(0,\tilde{r}_i\right)$, the last step holds because $\Theta_i^{t\tau+r'_i}$ is not full rank. Then, we proceed to bound $\left\|\Theta_i^{t\tau+\tilde{r}_i}-\Theta_i^0\right\|_{\rm op}$, which can be derived as
\begin{align}\label{eq:equation75}
	&\left\|\Theta_i^{t\tau+\tilde{r}_i}-\Theta_i^0\right\|_{\rm op}\nonumber\\
	&= \frac{1}{n}\left\|J\left(\theta_i^{t\tau+\tilde{r}_i}\right)^TJ_i\left(\theta_i^{t\tau+\tilde{r}_i}\right)-J\left(\theta_i^0\right)^TJ_i\left(\theta_i^0\right)\right\|_{\rm op}\nonumber\\
	&= \frac{1}{n}\left\|[J\left(\theta_i^{t\tau+\tilde{r}_i}\right)^T-J\left(\theta_i^0\right)^T]J_i\left(\theta_i^{t\tau+\tilde{r}_i}\right)+J\left(\theta_i^0\right)^T[J\left(\theta_i^{t\tau+\tilde{r}_i}\right)-J_i\left(\theta_i^0\right)]\right\|_{\rm op}\nonumber\\
	&\leq \frac{1}{n}\left\|J\left(\theta_i^{t\tau+\tilde{r}_i}\right)^T-J\left(\theta_i^0\right)^T\right\|_F\left\|J_i\left(\theta_i^{t\tau+\tilde{r}_i}\right)\right\|_F+\left\|J\left(\theta_i^0\right)^T\right\|_F\left\|J\left(\theta_i^{t\tau+\tilde{r}_i}\right)-J_i\left(\theta_i^0\right)\right\|_F\nonumber\\
	&\leq 2C^2\left\|\theta_i^{t\tau+\tilde{r}_i} - \theta^0\right\|_2\nonumber\\
	&\overset{\sf (a)}{\leq} \frac{2\eta_0 \tilde{r}_iC^3R_0 q^{t}}{\sqrt{n}|\mathcal{D}_i|}+ \frac{2\eta_0 \tau C^3R_0\left(1-q^t\right) }{\sqrt{n}\left(1-q\right)}\nonumber\\
	&\leq \frac{2\eta_0 \tau C^3 q^{t}R_0}{\sqrt{n}}+ \frac{2\eta_0 \tau C^3R_0\left(1-q^t\right) }{\sqrt{n}\left(1-q\right)},
\end{align}
where step (a) comes from \eqref{eq:equation56} and the last step holds because $\tilde{r}_i\leq\tau$ and $|\mathcal{D}_i|\geq1$. Plugging \eqref{eq:equation75} into \eqref{eq:equation76} yields
\begin{align}\label{eq:equation120}
	&\left\|\sum_{i=1}^{M}p_i\frac{\eta_0}{|\mathcal{D}_i|}e^{-\frac{\eta_0\tau}{|\mathcal{D}_i|}\Theta_i^0}\int_{0}^{\tau}e^{\frac{\eta_0r}{|\mathcal{D}_i|}\Theta_i^0}\left(\Theta_i^{t\tau+r}-\Theta_i^0\right)g\left(\theta_i^{t\tau+r}\right)dr\right\|_{\rm op}\nonumber\\
	&\leq \frac{M\eta_0\tau}{|\mathcal{D}|}\left[\frac{2\eta_0 \tau C^3 q^{t}R_0}{\sqrt{n}}+ \frac{2\eta_0 \tau C^3R_0\left(1-q^t\right) }{\sqrt{n}\left(1-q\right)}\right]\left\|g\left(\theta_i^{t\tau}\right)\right\|_{2}\nonumber\\
	&\overset{\sf (a)}{\leq} \frac{M\eta_0\tau}{|\mathcal{D}|}\left[\frac{2\eta_0 \tau C^3 q^{t}R_0}{\sqrt{n}}+ \frac{2\eta_0 \tau C^3R_0\left(1-q^t\right) }{\sqrt{n}\left(1-q\right)}\right]\left[q^{t}\left(R_0-\frac{2\eta_0\tau CC_1R_0\zeta}{\left(1-q\right)^2}\right) + \frac{2\eta_0\tau CC_1R_0\zeta}{\left(1-q\right)^2}\right]\nonumber\\
	&\overset{\sf(b)}{=} \frac{M\eta_0\tau}{|\mathcal{D}|}\left[\frac{2\eta_0 \tau C^3 q^{t}R_0}{\sqrt{n}}+ \frac{2\eta_0 \tau C^3R_0\left(1-q^t\right) }{\sqrt{n}\left(1-q\right)}\right]q^tR_0\nonumber\\
	&= \frac{2M\eta_0^2\tau^2 C^3R_0^2q^t\left(1-q^{t+1}\right)}{\sqrt{n}|\mathcal{D}|\left(1-q\right)}\nonumber\\
	&\leq \frac{2\eta_0^2\tau^2 C^3R_0^2q^t\left(1-q^{t+1}\right)}{\sqrt{n}\left(1-q\right)},
\end{align}
where step $\sf (a)$ comes from \eqref{eq:rate} in Theorem~\ref{theorem1} and step $(b)$ omits  $\frac{\zeta}{\sqrt{n}}=\mathcal{O}(n^{-1})$.

\subsection{Bounding $\left\|f^{\rm lin}\left(\theta^{t\tau+r}\right)-f\left(\theta^{t\tau+r}\right)\right\|_2$ and $\left\|f_i^{\rm lin}\left(\theta_i^{t\tau+r}\right)-f_i\left(\theta_i^{t\tau+r}\right)\right\|_2$}
Plugging \eqref{eq:equation120} and \eqref{eq:equation87} into \eqref{eq:equation58} yields
\begin{align}\label{eq:equation78}
	\left\|g^{\rm lin}\left(\theta^{\left(t+1\right)\tau}\right)-g\left(\theta^{\left(t+1\right)\tau}\right)\right\|_2
	&\leq q\left\|g^{\rm lin}\left(\theta^{t\tau}\right)-g\left(\theta^{t\tau}\right)\right\|_{2} +  \frac{2\eta_0^2\tau^2 C^3R_0^2q^t\left(1-q^{t+1}\right)}{\sqrt{n}\left(1-q\right)}+ \left\|\sum_{i=1}^{M}p_i\Omega_i\right\|_2\nonumber\\
	&\leq q\left\|g^{\rm lin}\left(\theta^{t\tau}\right)-g\left(\theta^{t\tau}\right)\right\|_{2} + \frac{2\eta_0^2\tau^2 C^3R_0^2q^t\left(1-q^{t+1}\right)}{\sqrt{n}\left(1-q\right)}+ \frac{2\eta_0\tau CC_1R_0\zeta}{\left(1-q\right)},
\end{align}
where the last step comes from \eqref{eq:equation300}. By recursively employing \eqref{eq:equation78} and considering the fact $\left\|g^{\rm lin}\left(\theta^{0}\right)-g\left(\theta^{0}\right)\right\|_2 = 0$, we obtain
\begin{align}\label{eq:equation79}
	\left\|g^{\rm lin}\left(\theta^{\left(t+1\right)\tau}\right)-g\left(\theta^{\left(t+1\right)\tau}\right)\right\|_2
	&\leq \frac{1-q^{t+1}}{1-q}\left[\frac{2\eta_0^2\tau^2 C^3R_0^2q^t\left(1-q^{t+1}\right)}{\sqrt{n}\left(1-q\right)}+ \frac{2\eta_0\tau CC_1R_0\zeta}{\left(1-q\right)}\right]\nonumber\\
	& = \frac{2\eta_0^2\tau^2 C^3R_0^2q^t\left(1-q^{t+1}\right)^2}{\sqrt{n}\left(1-q\right)^2}+ \frac{2\eta_0\tau CC_1R_0\zeta\left(1-q^{t+1}\right)}{\left(1-q\right)^2}.
\end{align}
Considering $\zeta = \frac{2\eta_0\tau CR_0}{\sqrt{n}\left(1-q\right)}$ and replacing $t+1$ with $t$, we have
\begin{align}
	\sup_{t\geq 0}\left\|g^{\rm lin}\left(\theta^{t\tau}\right)-g\left(\theta^{t\tau}\right)\right\|_2 = \sup_{t \geq 0}\left\|f^{\rm lin}\left(\theta^{t\tau}\right)-f\left(\theta^{t\tau}\right)\right\|_2 = \mathcal{O}\big(n^{-\frac{1}{2}}\big).
\end{align}

Finally, we bound $\left\|f_i^{\rm lin}\left(\theta_i^{t\tau+r}\right)-f_i\left(\theta_i^{t\tau+r}\right)\right\|_2$ as
\begin{align}
	\left\|f_i^{\rm lin}\left(\theta_i^{t\tau+r}\right)-f_i\left(\theta_i^{t\tau+r}\right)\right\|_2 
	&= \left\|g_i^{\rm lin}\left(\theta_i^{t\tau+r}\right)-g_i\left(\theta_i^{t\tau+r}\right)\right\|_2 \nonumber\\
	&= \left\|P_ig^{\rm lin}\left(\theta_i^{t\tau+r}\right)-P_ig\left(\theta_i^{t\tau+r}\right)\right\|_2\nonumber\\
	&\leq \left\|P_i\right\|_{\rm op}\left\|g^{\rm lin}\left(\theta_i^{t\tau+r}\right)-g\left(\theta_i^{t\tau+r}\right)\right\|_2\nonumber\\
	&=\left\|g^{\rm lin}\left(\theta_i^{t\tau+r}\right)-g\left(\theta_i^{t\tau+r}\right)\right\|_2\nonumber\\
	&\overset{\sf (a)}{\leq}  q\left\|g^{\rm lin}\left(\theta^{t\tau}\right)-g\left(\theta^{t\tau}\right)\right\|_{2} + \frac{2\eta_0^2\tau^2 C^3R_0^2q^t\left(1-q^{t+1}\right)}{\sqrt{n}\left(1-q\right)}\nonumber\\
	&\leq q\left[\frac{2\eta_0^2\tau^2 C^3R_0^2q^{t-1}\left(1-q^{t}\right)^2}{\sqrt{n}\left(1-q\right)^2}+ \frac{2\eta_0\tau CC_1R_0\zeta\left(1-q^{t}\right)}{\left(1-q\right)^2}\right] + \frac{2\eta_0^2\tau^2 C^3R_0^2q^t\left(1-q^{t+1}\right)}{\sqrt{n}\left(1-q\right)}\nonumber\\
	&=\frac{2\eta_0^2\tau^2 C^3R_0^2q^{t}\left(1-q^{t}\right)^2}{\sqrt{n}\left(1-q\right)^2}+ \frac{2\eta_0\tau CC_1R_0\zeta q\left(1-q^{t}\right)}{\left(1-q\right)^2} + \frac{2\eta_0^2\tau^2 C^3R_0^2q^t\left(1-q^{t+1}\right)}{\sqrt{n}\left(1-q\right)},
\end{align}
where step $\sf (a)$ holds because of \eqref{eq:equation55}, \eqref{eq:equation87} and \eqref{eq:equation120}. The third inequation holds because of \eqref{eq:equation79}. Further considering $\zeta = \frac{2\eta_0\tau CR_0}{\sqrt{n}\left(1-q\right)}$, we have
\begin{align}
	\sup_{t\geq0,~1\leq r \leq \tau} \left\|f_i^{\rm lin}\left(\theta_i^{t\tau+r}\right)-f_i\left(\theta_i^{t\tau+r}\right)\right\|_2 = \mathcal{O}\big(n^{-\frac{1}{2}}\big), ~\forall i
\end{align}
Thus, Theorem~\ref{theorem2} is proved.
\section{Proof of Theorem~\ref{theorem3}}\label{appendixD}

The Jacobians of the linear global and local models are
\begin{equation}
	J^{\rm lin}\left(\theta^{t\tau}\right)  = J\left(\theta^0\right), \quad 
	J_i^{\rm lin}\left(\theta_i^{t\tau+r}\right) = J_i\left(\theta_i^0\right).
\end{equation}
The local update process of the linear model can be expressed as
\begin{align}
	\theta_i^{t\tau+r+1} = \theta_i^{t\tau+r} - \eta J_i\left(\theta_i^0\right) g^{\rm lin}_i\left(\theta_i^{t\tau+r}\right).
\end{align}
Via continuous time gradient flow, we obtain
\begin{align}\label{eq:equation97}
	\frac{d\theta_i^{t\tau+r}}{dr}=-\frac{\eta}{|\mathcal{D}_i|} J_i\left(\theta_i^0\right) g^{\rm lin}_i\left(\theta_i^{t\tau+r}\right)=-\frac{\eta}{|\mathcal{D}_i|} J_i\left(\theta_i^0\right) P_ig^{\rm lin}\left(\theta_i^{t\tau+r}\right).
\end{align}
Employing the chain rule yields
\begin{align}\label{eq:equation98}
	\frac{dg^{\rm lin}\left(\theta_i^{t\tau+r}\right)}{dr} &= J\left(\theta_i^0\right)^T\frac{d\theta_i^{t\tau+r}}{dr} \nonumber\\
	&= -\frac{\eta}{|\mathcal{D}_i|} J\left(\theta_i^0\right)^TJ_i\left(\theta_i^0\right)P_ig^{\rm lin}\left(\theta_i^{t\tau+r}\right)\nonumber\\
	&=-\frac{\eta_0\Theta_i^0}{|\mathcal{D}_i|} g^{\rm lin}\left(\theta_i^{t\tau+r}\right).
\end{align}
Integrating from $0$ to $r$ on both sides yields
\begin{align}\label{eq:equation99}
	g^{\rm lin}\left(\theta_i^{t\tau+r}\right) = e^{\frac{-\eta_0r}{\mathcal{D}_i}\Theta_i^0}g^{\rm lin}\left(\theta^{t\tau}\right).
\end{align}
Plugging \eqref{eq:equation99} into \eqref{eq:equation97} yields
\begin{align}
	\frac{d\theta_i^{t\tau+r}}{dr}=-\frac{\eta}{|\mathcal{D}_i|} J_i\left(\theta_i^0\right)P_i e^{\frac{-\eta_0r}{\mathcal{D}_i}\Theta_i^0}g^{\rm lin}\left(\theta^{t\tau}\right).
\end{align}
Integrating from $0$ to $\tau$ on both sides yields
\begin{align}\label{eq:equation109}
	\theta_i^{t\tau+\tau} -\theta^{t\tau} = -\frac{\eta}{|\mathcal{D}_i|} J_i\left(\theta_i^0\right)P_i \left(\int_{0}^{\tau}e^{\frac{-\eta_0r}{\mathcal{D}_i}\Theta_i^0}dr\right) g^{\rm lin}\left(\theta^{t\tau}\right),
\end{align}
where $\int_{0}^{\tau}e^{\frac{-\eta_0r}{\mathcal{D}_i}\Theta_i^0}dr$ can be further derived as
\begin{align}\label{eq:equation110}
	\int_{0}^{\tau}e^{\frac{-\eta_0\Theta_i^0r}{\mathcal{D}_i}}dr
	&=\int_{0}^{\tau}\sum_{k=0}^{\infty}\frac{1}{k!}\left(-\frac{\eta_0r}{|\mathcal{D}_i|}\Theta_i^0\right)^kdr\nonumber\\
	&=\sum_{k=0}^{\infty}\frac{1}{k!}\left(-\frac{\eta_0\tau}{|\mathcal{D}_i|}\Theta_i^0\right)^k\frac{\tau}{k+1}\nonumber\\
	&=\frac{-|\mathcal{D}_i|}{\eta_0}\left(\Theta_i^0\right)^{-1}\sum_{k=0}^{\infty}\frac{1}{\left(k+1\right)!}\left(-\frac{\eta_0\tau}{|\mathcal{D}_i|}\Theta_i^0\right)^{k+1}\nonumber\\
	&=\frac{-|\mathcal{D}_i|}{\eta_0}\left(\Theta_i^0\right)^{-1}\sum_{k=1}^{\infty}\frac{1}{k!}\left(-\frac{\eta_0\tau}{|\mathcal{D}_i|}\Theta_i^0\right)^{k}\nonumber\\
	&=\frac{|\mathcal{D}_i|}{\eta_0}\left(\Theta_i^0\right)^{-1}\left(I-e^{-\frac{\eta_0\tau}{|\mathcal{D}_i|}\Theta_i^0}\right).
\end{align}
Plugging \eqref{eq:equation110} into \eqref{eq:equation109} yields
\begin{align}
	\theta_i^{t\tau+\tau} -\theta^{t\tau} 
	&= -\frac{\eta}{|\mathcal{D}_i|} J_i\left(\theta_i^0\right)P_i \frac{|\mathcal{D}_i|}{\eta_0}\left(\Theta_i^0\right)^{-1}\left(I-e^{-\frac{\eta_0\tau}{|\mathcal{D}_i|}\Theta_i^0}\right) g^{\rm lin}\left(\theta^{t\tau}\right)\nonumber\\
	&=-\frac{1}{n}\left(J\left(\theta^0\right)^T\right)^{-1}J\left(\theta^0\right)^TJ_i\left(\theta_i^0\right)P_i\left(\Theta_i^0\right)^{-1}\left(I-e^{-\frac{\eta_0\tau}{|\mathcal{D}_i|}\Theta_i^0}\right) g^{\rm lin}\left(\theta^{t\tau}\right)\nonumber\\
	&=-\left(J\left(\theta^0\right)^T\right)^{-1}\left(I-e^{-\frac{\eta_0\tau}{|\mathcal{D}_i|}\Theta_i^0}\right) g^{\rm lin}\left(\theta^{t\tau}\right).
\end{align}
Considering the model aggregation process, we have
\begin{align}\label{eq:equation104}
	\theta^{\left(t+1\right)\tau}-\theta^{t\tau} 
	&= \sum_{i=1}^{M}p_i\left(\theta_i^{t\tau+\tau} -\theta^{t\tau}\right)\nonumber\\
	&=-\left(J\left(\theta^0\right)^T\right)^{-1}\sum_{i=1}^{M}p_i\left(I-e^{-\frac{\eta_0\tau}{|\mathcal{D}_i|}\Theta_i^0}\right) g^{\rm lin}\left(\theta^{t\tau}\right)\nonumber\\
	&\overset{\sf(a)}{=}-\left(J\left(\theta^0\right)^T\right)^{-1}\left(I-e^{-\sum_{i=1}^{M}p_i\frac{\eta_0\tau}{|\mathcal{D}_i|}\Theta_i^0}\right) g^{\rm lin}\left(\theta^{t\tau}\right)\nonumber\\
	&=-\left(J\left(\theta^0\right)^T\right)^{-1}\left(I-e^{-\frac{\eta_0\tau}{|\mathcal{D}|}\Theta^0}\right) g^{\rm lin}\left(\theta^{t\tau}\right),
\end{align}
where step $(a)$ employs Lemma~\ref{lem:lemma4}. Considering the aggregation process of the linear model, we have
\begin{align}\label{eq:equation105}
	g^{\rm lin}\left(\theta^{\left(t+1\right)\tau}\right)
	&= \sum_{i=1}^{M} p_ig^{\rm lin}\left(\theta_i^{t\tau+\tau}\right)\nonumber\\
	&= \sum_{i=1}^{M} p_ie^{\frac{-\eta_0\tau}{\mathcal{D}_i}\Theta_i^0}g^{\rm lin}\left(\theta_i^{t\tau}\right)\nonumber\\
	&\overset{\sf (a)}{=} e^{-\eta_0\tau\sum_{i=1}^{M}\frac{p_i}{\mathcal{D}_i}\Theta_i^0}g^{\rm lin}\left(\theta_i^{t\tau}\right)\nonumber\\
	&=e^{-\frac{\eta_0\tau}{\mathcal{D}}\Theta^0}g^{\rm lin}\left(\theta^{t\tau}\right),
\end{align}
where step $\sf (a)$ employs Lemma~\ref{lem:lemma4}. By cursively employing \eqref{eq:equation105}, we can obtain
\begin{equation}
	g^{\rm lin}\left(\theta^{\left(t+1\right)\tau}\right) = e^{-\frac{\eta_0\left(t+1\right)\tau}{\mathcal{D}}\Theta^0}g^{\rm lin}\left(\theta^0\right).
\end{equation}
Replace $t+1$ with $t$ yields $	g^{\rm lin}\left(\theta^{t\tau}\right) = e^{-\frac{\eta_0\Theta^0t\tau}{\mathcal{D}}}g^{\rm lin}\left(\theta^0\right)$, plugging which  into \eqref{eq:equation104} further yields
\begin{align}\label{eq:equation125}
	\theta^{\left(t+1\right)\tau}-\theta^{t\tau}
	&= -\left(J\left(\theta^0\right)^T\right)^{-1}\left(I-e^{-\frac{\eta_0\tau}{|\mathcal{D}|}\Theta^0}\right) e^{\frac{-\eta_0t\tau}{\mathcal{D}}\Theta^0}g^{\rm lin}\left(\theta^0\right)\nonumber\\
	&=-\left(J\left(\theta^0\right)^T\right)^{-1}\left(e^{\frac{-\eta_0t\tau}{\mathcal{D}}\Theta^0}-e^{-\frac{\eta_0\left(t+1\right)\tau}{|\mathcal{D}|}\Theta^0}\right) g^{\rm lin}\left(\theta^0\right).
\end{align}
Employing \eqref{eq:equation125} iteratively yields
\begin{align}
	\theta^{\left(t+1\right)\tau}-\theta^{0} = -\left(J\left(\theta^0\right)^T\right)^{-1}\left(I-e^{-\frac{\eta_0\left(t+1\right)\tau\Theta^0}{|\mathcal{D}|}}\right) g^{\rm lin}\left(\theta^0\right).
\end{align}
Replacing $t+1$ with $t$ yields
\begin{align}
	\theta^{t\tau} - \theta^0 = -\frac{1}{n}J\left(\theta^0\right)\left(\Theta^0\right)^{-1}\left(I-e^{-\frac{\eta_0\Theta^0t\tau}{|\mathcal{D}|}}\right)g^{\rm lin}\left(\theta^0\right).
\end{align}
Therefore, for an arbitrary input $x$, we can obtain the closed form of $g^{\rm lin}\left( x, \theta^{t\tau}\right)$ and $f^{\rm lin}\left(x, \theta^{t\tau} \right)$ as
\begin{align}
	g^{\rm lin}\left(x, \theta^{t\tau}\right)
	&= g\left(x, \theta^0\right) + J\left(x, \theta^0\right)^T\left(\theta^{t\tau}-\theta^0\right)\nonumber\\
	&= g\left(x, \theta^0\right) - \Theta^0\left(x\right)\left(\Theta^0\right)^{-1}\left(I-e^{-\frac{\eta_0\Theta^0t\tau}{|\mathcal{D}|}}\right)g^{\rm lin}\left(\theta^0\right)
\end{align}
and
\begin{align}
	f^{\rm lin}\left(x, \theta^{t\tau}\right)= f\left(x, \theta^0\right) - \Theta^0\left(x\right)\left(\Theta^0\right)^{-1}\left(I-e^{-\frac{\eta_0\Theta^0t\tau}{|\mathcal{D}|}}\right)\left(f\left(\mathcal{X}, \theta^0\right)-\mathcal{Y}\right),
\end{align}
respectively.
\newpage
\section{Network Details} \label{sec:network}
\begin{table}[!htb]
	\caption{Fully-Connected Network}
	\label{FC Networks}
	\vskip 2mm
	\begin{center}
		\begin{small}
			\begin{sc}
				\begin{tabular}{lcc}
					\toprule
					Kayer & Input & Iutput\\
					\midrule
					FC1 & Input Size & $8\times k$ \\
					FC2 & $8\times k$ & $8\times k$\\
					FC3 & $8\times k$ & 
					Output Size \\
					\bottomrule
				\end{tabular}
			\end{sc}
		\end{small}
	\end{center}
	\vskip -0.1in
\end{table}
\begin{table}[!htb]
	\caption{Convolution Network}
	\label{Convolution Network}
	\vskip 2mm
	\begin{center}
		\begin{small}
			\begin{sc}
				\begin{tabular}{lcccc}
					\toprule
					Layer & Channel& Kernel & output & stride\\\midrule
					conv1 & $6\times k$ & $5\times5$ & $28\times28$ &$1$\\
					pooling & $6\times k$ & $2\times 2$ & $14\times14$ & $2$\\
					conv2 & $16\times k$ & $5\times5$ & $10\times10$ &$1$\\
					pooling & $16\times k$ & $2\times 2$ & $5\times5$ & $2$\\
					fc1 & $-$ & $-$ & $120\times k$ &$-$\\
					fc2 & $-$ & $-$ & $84\times k$ &$-$\\
					fc3 & $-$ & $-$ & $10$ &$-$\\
					\bottomrule
				\end{tabular}
			\end{sc}
		\end{small}
	\end{center}
	\vskip -0.1in
\end{table}
\begin{table}[!htb]
	\caption{Residual Network}
	\label{Residual Network}
	\begin{center}
		\begin{small}
			\begin{sc}
				\begin{tabular}{lcc}
					\toprule
					layer & output & block type \\
					\midrule
					conv1    & $32\times32$ & $\left[3 \times 3, \text{channel size} \times k \right]$ \\
					conv2    & $32\times32$ & 
					$\left[
					\begin{array}{c}
						3 \times 3, \text{channel size} \times k \\
						3 \times 3, \text{channel size} \times k
					\end{array}
					\right] \times \psi$ \\
					conv3    & $16\times16$ & 
					$\left[
					\begin{array}{c}
						3 \times 3, \text{channel size} \times k\\
						3 \times 3, \text{channel size} \times k
					\end{array}
					\right] \times \psi$ \\
					conv4    & $8\times8$ & 
					$\left[
					\begin{array}{c}
						3 \times 3, \text{channel size} \times k \\
						3 \times 3, \text{channel size} \times k
					\end{array}
					\right] \times \psi$ \\
					avg-pool & $1\times1$ & $\left[8 \times 8\right]$ \\
					\bottomrule
				\end{tabular}
			\end{sc}
		\end{small}
	\end{center}
	\vskip -0.1in
\end{table}

\end{document}